\documentclass[twoside,11pt]{article}

\usepackage[preprint]{jmlr2e}
\usepackage[utf8]{inputenc} 
\usepackage[T1]{fontenc}    
\usepackage{hyperref}       
\hypersetup{hidelinks}
\usepackage{url}            
\usepackage{booktabs}       
\usepackage{nicefrac}       
\usepackage{microtype}      
\usepackage{xcolor}         
\usepackage{amsmath,amsfonts}
\usepackage{mathtools}
\mathtoolsset{showonlyrefs}
\usepackage{bm}
\usepackage{subcaption}
\usepackage{algorithm}
\usepackage{bbm}
\usepackage{algpseudocode}
\usepackage{comment}
\usepackage{tikz}
\usepackage{enumerate}

\usepackage{tikz}
\usetikzlibrary{positioning}
\usepackage{pgfplots}
\pgfplotsset{compat=1.16}
\usepackage[export]{adjustbox}
\usetikzlibrary{spy}

\usepackage{geometry}

\newcommand{\R}{\mathbb{R}}
\newcommand{\N}{\mathbb{N}}
\renewcommand{\d}{\mathrm{d}}
\renewcommand{\P}{\mathcal{P}}

\newcommand{\E}{\mathbb E}

\newcommand{\zb}{\bm}

\DeclareMathOperator{\Id}{Id}

\DeclareMathOperator{\supp}{supp}

\usepackage{lastpage}
\jmlrheading{26}{2025}{1-\pageref{LastPage}}{4/24; Revised
6/25}{8/25}{24-0586}{Jannis Chemseddine, Paul Hagemann, Gabriele Steidl, Christian Wald}

\ShortHeadings{Conditional Wasserstein Distances}{Chemseddine, Hagemann, Steidl and Wald}
\firstpageno{1}

\begin{document}

\title{Conditional Wasserstein Distances with Applications in Bayesian OT Flow Matching}

\author{\name Jannis Chemseddine \thanks{Alphabetical Ordering} \email chemseddine@math.tu-berlin.de\\
       \name Paul Hagemann  \email hagemann@math.tu-berlin.de\\
       \name Gabriele Steidl \email steidl@math.tu-berlin.de \\
       \name Christian Wald \email wald@math.tu-berlin.de \\
       \addr Institute of Mathematics\\
       Technical University of Berlin\\
       10623 Berlin, Germany}

\editor{Aryeh Kontorovich}

\maketitle

\begin{abstract}%
  In inverse problems, many conditional generative models approximate the posterior measure by minimizing a distance between the joint measure and its learned approximation. While this approach also controls the distance between the posterior measures in the case of the Kullback--Leibler divergence, the same in general does
  not hold true for the Wasserstein distance. 
  In this paper, we introduce a conditional Wasserstein distance via a set of restricted couplings that equals the expected Wasserstein distance of the posteriors.
  Interestingly, the dual formulation of the conditional Wasserstein-1 distance resembles losses
  in the conditional Wasserstein GAN literature in a quite natural way.
  We derive theoretical properties of the conditional Wasserstein distance, 
  characterize the corresponding geodesics and velocity fields as well as the flow ODEs. 
  Subsequently, we propose to approximate the  velocity fields by relaxing the conditional Wasserstein distance. Based on this, we propose an extension of OT Flow Matching for solving Bayesian inverse problems  and demonstrate its numerical advantages on an inverse problem and class-conditional image generation.
\end{abstract}


\section{Introduction}
Many sampling algorithms for the posterior $P_{X|Y=y}$ 
in Bayesian inverse problems
\begin{equation} \label{inverse}
Y = f(X) + \Xi
\end{equation}
with a forward operator $f: \mathbb X \to \mathbb Y$,
and a noise model $\Xi$,
perform learning either implicitly or explicitly on the joint distribution $P_{Y,X}$.
Most approaches minimize (or upper bound) some loss of the form 
$$L(\theta) = D(P_{Y,X},P_{Y,G_{\theta}}),$$
where $D$ denotes a suitable distance on the space of probability measures. In this framework $G_{\theta}$ is a conditional generative model, which in particular also depends on $y$.
For instance, this is done in the framework of conditional (stochastic) normalizing flows \citep{ardizzone2019guided,HHS22,HHG2023,winkler2019learning}, conditional GANs \citep{mirza2014conditional} or conditional gradient flows for the Wasserstein metric \citep{du2023nonparametric,hagemann2023posterior}. 
In practice however, we mostly look for the posteriors for single points $\tilde{y}$.
Recently in \citep{altekrueger2023conditional}, it was shown that the posterior error for single points $\tilde{y}$ goes to zero if the expected error to the posterior $\mathbb{E}_{Y}\left[W_1(P_{X|Y=y}, P_{Z|Y=y}) \right]$ is small in \citep{altekrueger2023conditional}, where $W_1$ denotes the Wasserstein distance \citep{villani2009optimal}. This shows that it is important to investigate the relation between $\mathbb{E}_{Y}\left[W_1(P_{X|Y=y}, P_{Z|Y=y}) \right]$ and $W_1(P_{Y,Z},P_{Y,X})$.

In \citep{CondWasGen}, the authors investigated the relation between the distance of the joint measures $D(P_{Y,Z}, P_{Y,X})$ and the expected error of the posteriors $\mathbb{E}_{Y}\left[D(P_{Z|Y=y}, P_{X|Y=y}) \right].$ For the Kullback--Leibler divergence $D= \text{KL}$, it follows by the chain rule 
\citep[Theorem 2.5.3]{cover_inf} that 
 $$\mathbb{E}_{y\sim P_Y}\left[\mathrm{KL}(P_{X|Y=y}, P_{Z|Y=y}) \right] = \mathrm{KL}(P_{Y,X}, P_{Y,Z}).$$  
Such results show that it is possible to approximate the averaged posterior via  the joint distribution. 
Unfortunately, we have for the Wasserstein-1 distance that in general only
\begin{equation} \label{w1_coupling}
W_1(P_{Y,X}, P_{Y,Z}) \le \mathbb{E}_{y\sim P_Y}\left[W_1(P_{X|Y=y}, P_{Z|Y=y})\right]
\end{equation}
holds true, in contrast to the equality claim in  \citep[Theorem 2]{CondWasGen}.
A simple counterexample is given in Appendix \ref{sec:ex}.
Intuitively, strict inequality can arise when the optimal transport (OT) plan needs to transport mass in the $Y$-component. 
This is the motivation for considering only plans that do not have mass transport in the $Y$-component.
This leads us to the definition of conditional Wasserstein distances $W_{p,Y}$, where  
admissible transport plans are restricted to the set 
$\Gamma_Y^4=\Gamma_Y^4(P_{Y,X},P_{Y,Z})$ 
of  4-plans 
$\alpha$ 
fulfilling in addition
$(\pi^{1,3})_{\sharp}\alpha = \Delta_{\sharp}P_{Y}$,
where $\Delta(y) = (y,y)$ is the diagonal map:
 \begin{align*}
  W_{p,Y}^p \coloneqq    \underset{\alpha \in \Gamma_Y^4}{\mathsf{inf}}\int\|(y_1,x_1)-(y_2,x_2)\|^p \, \d \alpha
 \end{align*}
 
Inspired by \citep{CondWasGen}, we show that this conditional Wasserstein distance indeed fulfills 
$$W_{p,Y}^p(P_{Y,X},P_{Y,Z}) = \mathbb{E}_{y\sim P_Y}\left[W_p^p(P_{X|Y=y}, P_{Z|Y=y})\right].$$

Further, we prove results on geodesics similar as in 
\citep{ambrosio2005gradient} for the conditional Wasserstein distance: we show the connections to the continuity equation, verify that there exists a velocity field with no mass transport in $Y$-direction and recover a corresponding ODE formulation.
Indeed, this conditional Wasserstein distance can be used to explain a numerical observation made by \citep{du2023nonparametric, hagemann2023posterior}, namely that rescaling the $Y$-component leads to velocity fields with no mass transport in Y-direction in the limit. Using these ideas, we propose to relax the conditional Wasserstein distance to allow "small amounts" of mass transport in $Y$-direction.

Then, we use our insights to design efficient posterior sampling algorithms. 
By leveraging recent ideas of flow matching, see \citep{albergo2023building, liu2023flow,lipman2023flow}, we design Bayesian OT flow matching. 
Note that the recent approaches of \citep{zheng2023guided, wildberger2023flow} do not respect the OT in $X$-direction as they always choose a random coupling between $x$ and $z$ for each observation $y$. This leads to unfortunate situations, where the optimal $Y$-diagonal coupling is not recovered even between Gaussians when we approximate them by empirical measures, see Example \ref{ex:emp_lim}.  
We use our proposed Bayesian OT flow matching and verify its advantages on a Gaussian mixture toy problem and on class conditional image generation on the CIFAR10 dataset. 

\textbf{Contributions}
\begin{itemize}
\item We introduce  conditional Wasserstein distances and highlight their relevance to conditional Wasserstein GANs in inverse problems. 
\item We derive theoretical properties of the conditional Wasserstein distance and  
establish geodesics in this conditional Wasserstein space, with velocity fields having no transport in $Y$-direction.
\item 
We show that the conditional Wasserstein distance can be used in conditional generative approaches and demonstrate the advantages on MNIST particle flows \citep{hagemann2023posterior, altekrueger23a}.  We propose a version of OT flow matching \citep{tong2023improving, pooladian2023multisample} for inverse problems  which uses a relaxed version of our conditional Wasserstein distance, and show that it overcomes obstacles, explained in Example \ref{ex:emp_lim}, from previous flow matching versions for inverse problems \citep{wildberger2023flow, albergo2024stochastic}.
\end{itemize}

\indent\textbf{Related work}
 Our work is in the intersection of conditional generative modeling \citep{adler2018deep, ardizzone2019guided, mirza2014conditional} and (computational) OT \citep{MAL-073, villani2009optimal}. 
The recent work \citep[Theorem 2]{kim2023wasserstein} derives an inequality based on restricting the admissible couplings in the their OT formulation to so-called conditional sub-couplings. Note that their reformulation is only a restatement of the expected value, but does not relate it to the joint distributions. Those authors also look for geodesics in the Wasserstein space, but pursue a different approach. While we relate it to the velocity fields in  gradient flows,  they pursue an autoencoder/GAN idea. 
To the best of our knowledge, the first work which considered conditional Wasserstein distances was \cite{kloeckner2021extensions}. Related absolutely continuous curves were discussed in \cite{peszek2023heterogeneous} including the existence of vector fields for absolutely continuous curves.  \cite{peszek2023heterogeneous} were mainly interested in absolutely continuous curves stemming from gradient flows and not in geodesics.
The closest work, which appeared after our first version of this paper, is  \citep{hosseini2024conditional}. 
Here the authors define the conditional optimal transport problem and calculate its dual. Their work is more focused on the infinite dimensional setting, whereas we consider the velocity fields needed for the flow matching application.
There are two other concurrent works that treat objects similar to what is presented in this paper. \citep{barboni2024understanding} study gradient flows for conditional Wasserstein spaces in the case where $P_Y$ is the Lesbesgue measure on $[0,1]$ in order to analyse the training of infinitely deep and wide ResNets. Another preprint appearing after the ArXiv submission of the present paper, which also uses conditional OT for flow matching, is \citep{kerrigan2024dynamic}. They come up with a similar loss function, which also uses that the velocity field should not transport mass in the Y-component. Another concurrent work is \citep{isobe2024extended}, where they use a generalized continuity equation to extend the flow matching framework to the matrix valued case, where they use it for style transfer. Theoretically our paper is more focused on finding geodesics where no mass in $Y$ is transported whereas they look for translations between classes. 

In the OT literature, there has been a collection of class conditional OT distances used in domain adaption \citep{nguyen2022cycle,rakotomamonjy2021optimal}. In particular, conditional OT as in \citep{tabak2021data} is relevant as they consider OT plans for each condition $y$ minimizing $\mathbb{E}_y[W_1(P_{X|Y=y}, G(\cdot,y)_{\#}P_Z)]$. However they relax their problem using a KL divergence.
The works on Wasserstein gradient flows \citep{ambrosio2005gradient, gig} investigate conditional Wasserstein distances  from a different point of view for defining the so-called geometric tangent space of the 2-Wasserstein space.  Geometric tangent spaces play a crucial role in
Wasserstein gradient flows of maximum mean discrepancies with Riesz kernels in \citep{HGBS2024} and their neural variants in \citep{altekrueger23a}.
In \citep[Remark 7]{hagemann2023posterior}, an inequality between the joint Wasserstein and the expected value over the conditionals is derived, where the result requires compactly supported measures and certain regularity of the associated posterior densities. In \citep{bunne2022supervised}, the supervised training of conditional Monge maps is proposed, for which the authors solved the dual problem using convex neural networks. The authors of \citep{manupriya2023empirical} also consider an amortized objective between the conditional distributions and propose a relaxation, which only needs  samples from the joint distribution involving  maximum mean discrepancies. 
Numerically, we first verify our theoretical statements based on particle flows, which were also used in \citep{altekrueger23a, hagemann2023posterior}. 
Further, we apply our framework to solve inverse problems using Bayesian flow matching \citep{wildberger2023flow, zheng2023guided} 
and  OT flow matching \citep{albergo2024stochastic,lipman2023flow, liu2023flow, tong2023improving, pooladian2023multisample}. 
\\[1ex]

This paper is an extension of our first ArXiv version \citep{CHW2023} 
on conditional Wasserstein distances with more focus on the continuity equation and flow equation for geodesics as well as flow matching. 

\textbf{Outline of the paper}
In Section \ref{sec:prelim}, we recall 
preliminaries from OT. Then, in Section \ref{sec:condW}, we introduce
conditional Wasserstein distances of joint probability measures, 
and show their relation
to the expectation over the Wasserstein distance of the conditional probabilities.
Moreover, we highlight the connection to work on
geometric tangent spaces.
In Section \ref{sec:gan}, we calculate the dual of our conditional Wasserstein-1 distance and show how a loss function used in the
conditional Wasserstein GAN literature arises in a natural way. 
In Section \ref{sec:velocity}, we deal with geodesics with respect to the conditional Wasserstein distance, 
prove properties of the corresponding velocity fields, showing that they vanish in the $Y$-component, and show existence for flow ODEs. 
We propose a relaxation of the conditional Wasserstein distance which appears to be useful for numerical computations
in Section \ref{sec:relax}.
We combine our findings with  OT flow matching to get Bayesian flow matching in Section 
\ref{sec:fm}.
Finally, in Section \ref{sec:numerics}, we present numerical results:  we verify a convergence result for an approximation of the conditional Wasserstein distance using particle flows to MNIST, and
demonstrate the advantages of our Bayesian OT flow matching procedure 
on a Gaussian mixture model toy example and on CIFAR10 class-conditional image generation.
All proofs are postponed to the appendices.
 \section{Preliminaries} \label{sec:prelim}
Throughout this paper, we will use the following notation. These are basics from from optimal transport theory and can be found in \citep{villani2009optimal}.
By 
$\mathcal{P}(\mathbb X)$,
we denote the set of probability measures on $\mathbb X \subseteq \mathbb R^n$ and by 
$\mathcal{P}_p(\mathbb X)$, $p \in [1,\infty)$ 
the subset of measures with finite $p$-th moments.
For  $\mu\in \mathcal P(\mathbb X)$ and a measurable function $F:\mathbb X \to \mathbb Y$, we define the \emph{push forward measure}  by $F_{\sharp}\mu = F \circ \mu^{-1}$. 
For a product space $\prod_{i=1}^K \mathbb X_i$, 
we denote the projection onto the $i_1,\ldots,i_k$-th component by $\pi^{i_1,\ldots,i_k}$. 
The \emph{Wasserstein-$p$ metric} \citep{villani2009optimal} on $\mathcal{P}_p(\mathbb X)$ is given by 
\begin{align} \label{eq:wasserstein}
W_p(\mu, \nu) &\coloneqq \Big(\min_{\gamma \in \Gamma} \int_{\mathbb X^2} \Vert x-y \Vert^p \, \text{d} \gamma(x,y)\Big)^{\frac 1 p}\\
&= \Big(\min_{\gamma \in \Gamma} \mathbb E_{(x,y) \sim \gamma} \big[ \|x-y\|^p \big] \Big)^{\frac 1 p}.
\end{align}
where $\Gamma=\Gamma(\mu,\nu)$ denotes the set of all probability measures 
$\gamma\in \P(\mathbb X \times \mathbb X)$ with marginals $\pi^1_{\sharp}\gamma=\mu$ and $\pi^2_{\sharp}\gamma =\nu$
and $\| \cdot\|$ is the Euclidean distance on $\mathbb R^n$, see \citep{villani2009optimal}.
If $\mu \in \mathcal P_p(\mathbb X)$ is absolutely continuous, then, for $p \in (1,\infty)$, there exists a unique optimal
transport map $T \in L^p_\mu(\mathbb X, \mathbb X)$, also known as Monge map,
which solves
$$
\min_{T \, \text{measurable}} \Big\{\int_{\mathbb X} \Vert x-T(x) \Vert^p \, \text{d} \mu(x) 
\quad \text{such that} \quad T_{\sharp} \mu = \nu \Big \}.
$$
Further, this optimal map is related to the optimal transport plan $\gamma$ in \eqref{eq:wasserstein} by
$
\gamma = (\mathrm{Id}, T)_{\sharp}\mu,
$
see \citep{villani2009optimal}. The same holds true for empirical measures with the same number of points, see \citep[Proposition 2.1]{MAL-073}.
In this paper, we ask for relations between joint and posterior probabilities:
for random variables $X,Z \in B \subseteq \mathbb R^m$ and 
$Y \in A \subseteq \mathbb R^d$, we are interested
in Wasserstein distances between 
$P_{Y,X}, P_{Y,Z} \in \mathcal P_p(A \times B)$
and $P_{X|Y=y}, P_{Z|Y=y} \in \mathcal P_p(B)$.
Since $\pi^1_\sharp P_{Y,X} = P_Y$ as well as $\pi^1_\sharp P_{Y,Z} = P_Y$,
we see that the joint probabilities belong indeed to the subset
\begin{equation} \label{velocity_space}
\P_{p,Y}(A\times B) \coloneqq \{ \gamma\in \P_p(A\times B): \pi^1_{\sharp}\gamma =P_Y\}.
\end{equation}
For $p=2$, this set was considered as set of  velocity plans at $P_Y$ in
\citep[Sect. 12.4]{ambrosio2005gradient} and \citep[Sect. 4]{gig}. It is the basis for defining the so-called geometric tangent space of $\mathcal P_2(\R^ d)$
which  was used, e.g. in  \citep{HGBS2024,altekrueger23a}  for handling
(neural) Wasserstein gradient flows of maximum mean discrepancies.

We will frequently apply the disintegration formula \citep[Theorem 5.3.1]{ambrosio2005gradient} which says 
that
for a measure $\gamma \in \mathcal P(A\times B)$ with $\pi^1_{\sharp} \gamma = \mu \in \mathcal P(A)$, there exists a $\mu$-a.e. uniquely determined Borel family  of probability measures $(\gamma_y)_{y \in A}$ such that
\begin{equation} \label{disint}
\int_{A\times B} f(y,x) \, \d \gamma(y,x) = \int_A \int_B f(y,x) \, \d\gamma_y(x) \d \mu(y)
\end{equation} 
for any Borel measurable map $f:A \times B \to [0,+\infty]$.
In particular, for $\gamma = P_{Y,X} \in \mathcal P(A\times B)$, the disintegration formula reads as 
\begin{equation} \label{disint_1}
\int_{A\times B} f(y,x) \, \d P_{Y,X}(y,x) = \int_A \int_B f(y,x) \, \d P_{X|Y=y}(x) \d P_Y(y).
\end{equation} 

 \section{Conditional Wasserstein Distance} \label{sec:condW}
As demonstrated in Appendix \ref{sec:ex} we can only expect inequality in \eqref{w1_coupling}.
Towards equality, we introduce a conditional Wasserstein distance 
which allows only couplings which leave the $Y$-component invariant. 
To this end, we introduce the set of special 4-plans
\begin{equation}\label{y-coupling}
  \Gamma_Y^4 =  \Gamma_Y^4(P_{Y,X},P_{Y,Z}) \coloneqq \Big\{ \alpha \in \Gamma(P_{Y,X},P_{Y,Z}): \pi^{1,3}_{\sharp}\alpha = \Delta_{\sharp}P_{Y} \Big\},
\end{equation}
 where $\Delta:A\to A^2$, $y \mapsto (y,y)$ is the diagonal map. 
 Note that $\Delta^{-1} (y_1,y_2) = \emptyset$ if $y_1 \not = y_2$
 and $\Delta^{-1} (y_1,y_2) = y$ if $y_1 = y_2 = y$.
 Then, we define the \emph{conditional Wasserstein-$p$ distance}, $p \in [1,\infty)$ by
        \begin{align}   \label{definition}
        W_{p,Y}(P_{Y,X},P_{Y,Z}) \coloneqq  \Big( 
        \inf_{\alpha \in \Gamma_Y^4} \int_{(A\times B)^2}\|(y_1,x_1)-(y_2,x_2)\|^p \, \d \alpha 
        \Big)^{\frac 1 p}.
        \end{align}    
Indeed we will see in Corollary \ref{cor:metric} that this is a metric on   $\P_{p,Y}(A\times B)$.  

In terms of Monge maps, this means that we are considering functions $(\mathrm{Id}, T(y,\cdot)): (y,x)\mapsto (y,T(y,x))$, where $T:\R^d\times \R^m\to\R^m$ and $(\mathrm{Id}, T(y,\cdot))_{\#}P_{Y,X} = P_{Y,Z}$. The following proposition gives the desired equivalence of the form \eqref{w1_coupling}.
The proof is given in Appendix \ref{app:condw}.

\begin{proposition}\label{cond:plan}
The following relations holds true.
\begin{itemize}
    \item[i)]  The conditional Wasserstein-$p$ distance \eqref{definition} fulfills
        \begin{align}\label{relation}
               W_{p,Y}^p(P_{Y,X},P_{Y,Z})= \mathbb E_{y\sim P_Y} \big[ W_p^p(P_{X|Y=y},P_{Z|Y=y})\big].
        \end{align}
\item[ii)] Let $\alpha \in \Gamma_{Y}^4$ be an optimal plan in \eqref{definition} with disintegration $\alpha_{y_1,y_2}$ with respect to $\pi^{1,3}_\sharp \alpha$. Then
  $\alpha_{y,y} \in \mathcal P(B^2)$ is
  an optimal plan for $W_p(P_{X|Y=y}, P_{Z|Y=y})$ for $P_Y$-a.e. $y \in A$.       
\item[iii)] There exists a collection of optimal plans 
$\alpha_y\in \Gamma(P_{X|Y=y},P_{Z|Y=y})$, $y \in A$ for 

$W_p(P_{X|Y=y},P_{Z|Y=y})$ such that 
 \begin{equation} \label{def_alpha}
 \alpha \coloneqq \int_{A} \d \delta_{y_1}(y_2) \, \d \alpha_{y_1}(x_1,x_2) \d P_Y(y_1)
 \end{equation}
is a well-defined  coupling in $\Gamma_Y^4$ which is   optimal  in \eqref{definition}.
    \end{itemize}
  \end{proposition}

Note that the definition of $\alpha$ in $iii)$ already appears in the proof of \citep[Theorem 2]{CondWasGen}. For $p=2$, it was shown in \citep[Sect. 12.4]{ambrosio2005gradient} and \citep[Sect. 4]{gig} that the square root of the right-hand side in \eqref{relation}
is a metric on $\mathcal P_{2,Y}(A \times B)$.
The proof can be generalized in a straightforward way for $p\in[1,\infty)$. 
Thus, by Proposition \ref{cond:plan}, we have the following corollary.

\begin{corollary} \label{cor:metric}
The conditional Wasserstein distance $W_{p,Y}$ is a metric on $\P_{p,Y}(A\times B)$.
\end{corollary}

Interestingly, for $p=2$, there was also given an equivalent definition by \citep{gig} of $\mathcal W_{p,Y}$,
namely
\begin{equation} \label{ew_1}
\mathcal W_{p,Y} (P_{Y,X}, P_{Y,Z}) \coloneqq
\underset{\beta\in \Gamma^3_{Y}(P_{Y,X},P_{Y,Z})} {\inf} \ 
\Big(\int_{A\times B^2}\|x_1 -x_2\|^p \, \d \beta(y,x_1,x_3) \Big)^{\frac 1 p}
\end{equation}
with the set of 3-plans
$$
\Gamma^3_{Y}(P_{Y,X},P_{Y,Z})
\coloneqq
\{\beta\in\mathcal{P}_p(A \times B^2 ): \pi^{1,2}_{\sharp}\beta = P_{Y,X},\pi^{1,3}_{\sharp}\beta=P_{Y,Z}\}.
$$
The relation between the admissible 3-plans and 4-plans is given in the following proposition, for which a proof can be found in the appendix. 

\begin{proposition} \label{prop:3-4}
The map
$
    \pi^{1,2,4}_{\sharp}: \Gamma_Y^4(P_{Y,X},P_{Y,Z})\to \Gamma_Y^3 (P_{Y,X},P_{Y,Z})
$
is a bijection and for every $\alpha\in \Gamma_Y^4(P_{Y,X},P_{Y,Z})$ it holds 
    \begin{align*}
         &\int_{(A\times B)^2}\|(y_1,x_1)-(y_2,x_2)\|^p\, \d \alpha=\int_{A\times B^2}\|x_1-x_2\|^p \, \d \pi^{1,2,4}_{\sharp}\alpha.    
    \end{align*}
\end{proposition}

\section{Dual Formulation of $W_{1,Y}$ and Relation to GAN Loss} \label{sec:gan}
Interestingly, the conditional Wasserstein-1 distance recovers loss functions
in the conditional Wasserstein GAN literature \citep{adler2018deep, CondWasGen, martin2021exchanging}. 
Wasserstein GANs \citep{arjovsky2017wasserstein} aim to sample from a target distribution $P_X$ based on a simpler distribution $P_Z$. A generator $G = G_\theta$ is learned such that the Wasserstein-1 distance
in its dual formulation
\begin{align*}
W_1(P_X, G_\# P_Z) = \max_{f\in \text{Lip}_1} \left\{\mathbb{E}_{X}[f]- \mathbb{E}_{G_\# P_Z}\big[f] \right\}=
\max_{f\in \text{Lip}_1} \left\{\mathbb{E}_{X}[f]- \mathbb{E}_{Z}\big[f\circ G \big] \right\}
\end{align*}
is minimized, where $\text{Lip}_1$ denotes the set of 1-Lipschitz continuous functions. 
At the same time, a discriminator $f = f_\omega$ is learned such that
the final Wasserstein GAN learning problem becomes
\begin{align*}
 \min_\theta \max_{\omega} \left\{\mathbb{E}_{X} [f] - \mathbb{E}_{Z} [f \circ G] \right\}
 \quad \text{subject to} \quad f \in \text{Lip}_1 .
\end{align*}
Usually, the 1-Lipschitz condition is enforced via so-called weight-clipping \citep{arjovsky2017wasserstein}.

In \citep{adler2018deep}, this approach was generalized to inverse problems. 
Assume that $A \subset \R^d$ and $B \subset \R^m$ are compact sets throughout this section.
For given $y \in A$, let $h(y,\cdot) \in \text{Lip}_1$ be a minimizer in
\begin{align*}
W_1(P_{X|Y=y}, G(y, \cdot)_\# P_Z) = 
\max_{h(y,\cdot) \in \text{Lip}_1} \left\{\mathbb{E}_{X|Y=y}[h(y,x)]- \mathbb{E}_{Z }\big[h(y,G(y,\cdot)) \big] \right\}.
\end{align*}
Now the authors take the expectation value 
on both sides and exchange expectation and maximum to get, together with \eqref{disint_1}, the relation
\begin{align} \label{wish}
\mathbb{E}_{Y}[W_1(P_{X|Y=y}, G(y, \cdot) _\# P_Z] 
= \max_{h}  \left\{
\mathbb{E}_{Y,X } [h] 
-
\mathbb{E}_{Y,Z} [h (y, G(y,\cdot)]
 \right\},
\end{align} 
where the maximum is taken over functions $h$ which are Lipschitz-1 continuous in the second variable.
However, exchanging maximum and expectation value requires that 
$(y,x) \mapsto h(y,x)$ is measurable which is not immediate.
This ,,gap'' was fixed under stronger assumptions, e.g. on the posterior, in \citep{martin2021exchanging}.

In this section, we show that the dual formulation of the conditional Wasserstein distance $W_{1,Y}$ leads naturally to the desired loss on the right-hand side of \eqref{wish} for an appropriate regular function set for $h$.
More precisely, we have the following theorem which is proved in Appendix \ref{sec:dual}. Note that we can give the precise space $\mathcal{F}$, where the functions we take the supremum over belong to.

\begin{theorem}\label{prop:gan}
    Let $A \subset \R^d$ and $B \subset \R^m$ be compact sets. Then it holds
     \begin{align*}
         W_{1,Y}(P_{Y,X},P_{Y,Z}) =  \sup_{h \in \mathcal F}\left\{ \mathbb{E}_{Y,X}[h] - \mathbb{E}_{Y,Z}[h]\right\},
     \end{align*}
     where $\mathcal F$ denotes the set of bounded, upper semi-continuous functions $h:A\times B\to \R$ satisfying 
     $|h(y,x_1)-h(y,x_2)|\leq \|x_1-x_2\|$
     for all $y \in A$ and all $x_1,x_2\in B$.
\end{theorem}

\section{Geodesics and Velocity Fields} \label{sec:velocity}
In this section, we deal with geodesics and velocity fields
in $\left( \mathcal P_{Y,2}(\mathbb R^d \times \R^m),W_{2,Y} \right)$.
We restrict our attention to   $p=2$ and $A = \mathbb R^d$, $B = \R^m$.
Coming from inverse problems, we have considered probability measures $P_{Y,X}$ related to random variables $(Y,X) \in \R^d \times \R^m$. When switching to flows, it is more convenient to address equivalently just probability measures on $\R^d \times \R^m$.

Let us recall some results, which can be found, e.g. in \citep{ambrosio2005gradient} for our setting.
A curve $\mu \colon [0,1] \to \P_2(\R^d \times \R^m)$ is called a \emph{geodesic} if 
\begin{equation}    \label{eq:geodesic}
    W_2(\mu_s, \mu_t) = |s - t| W_2(\mu_0,\mu_1) \quad \text{for all} \quad s, t \in [0,1].
\end{equation}
The Wasserstein space is geodesic, 
i.e.\ any two measures $\mu_0, \mu_1 \in \P_2(\R^d \times \R^m)$ 
can be connected by a geodesic. 
Let $e_t:(\R^d\times \R^m)^2 \to \R^n\times \R^m$, $t \in [0,1]$ be defined by
\begin{equation} \label{e_t}
e_t(y_1,x_1,y_2,x_2) \coloneqq 
\left( (1-t)\pi^{1,2} + t \pi^{3,4} \right)(y_1,x_1,y_2,x_2) 
= (1-t)(y_1,x_1) + t(y_2,x_2) .
\end{equation}
Any geodesic $\mu \colon [0,1] \to \P_2(\R^d \times \R^m)$ connecting $\mu_0, \mu_1 \in \P_2(\R^d \times \R^m)$  is determined by an optimal plan $\alpha \in \Gamma(\mu_0, \mu_1)$ in \eqref{eq:wasserstein} via
    \begin{equation} \label{eq:geodesic_plan}
      \mu_t \coloneqq (e_t)_{\sharp} \alpha, \quad t \in [0,1].
    \end{equation}
Conversely, any optimal plan $\alpha \in \Gamma(\mu_0, \mu_1)$ gives by \eqref{eq:geodesic_plan} rise to a geodesic connecting $\mu_0$ and $\mu_1$. The following lemma considers curves defined by  \eqref{eq:geodesic_plan} 
which connect measures $\mu_0, \mu_1 \in \P_{2,Y}(\R^d \times \R^m)$. 
Its proof is given in the appendix and is similar to \citep[Theorem 7.2.2]{ambrosio2005gradient}.

\begin{lemma}\label{plan:geo_dis}
Let $\mu_0,\mu_1\in \P_{2,Y}(\R^d\times \R^m)$ and 
let $\alpha\in \Gamma_{Y}^4(\mu_0,\mu_1)$ be an optimal plan in \eqref{definition}. 
Then the following holds true.
\begin{itemize}
\item[i)]
The curve
$
\mu_t \coloneqq (e_t)_{\sharp}\alpha
$
is a geodesic in $(\P_{2,Y}(\R^d \times \R^m), W_{2,Y})$.
\item[ii)] The curve  
$(\mu_t)_{y}:=(1-t)\pi^1+t\pi^2)_\sharp\alpha_{y,y}$ is a disintegration of $\mu_t$ with respect to $P_Y$. Further,  
$(\mu_t)_{y}$ is a geodesic in $(\P_2(\R^m), W_2)$ for $P_Y$-a.e. $y\in\R^d$.
\item[iii)] $\mu_t$ is weakly continuous.
\end{itemize}
\end{lemma}

By the following proposition, the above geodesic $\mu_t$ has an associated vector field $v_t$ such that $(\mu_t,v_t)$ satisfy a continuity equation. Moreover, informally speaking, 
the associated vector field $v_t$
does not transport any mass in the $y$-component. This is related to the observation in \citep[Section 4.2]{du2023nonparametric} and \cite[Proposition 3.21]{peszek2023heterogeneous}.

\begin{proposition}\label{prop:continuity}
Let $\mu_0,\mu_1\in \P_{2,Y}(\R^d\times \R^m)$. 
Let $\alpha\in \Gamma_{Y}^4(\mu_0,\mu_1)$ be an optimal plan in \eqref{definition}
and $\mu_t = (e_t)_\sharp \alpha$, $t \in [0,1]$.
Then there exists a Borel measurable vector field $v:[0,1]\times \R^{d}\times \R^m\to \R^d\times \R^m,\, v(t,y,x)=v_t(y,x)$ with $v_t\in L^2_{\mu_t}(\R^{d} \times \R^m,\R^{d} \times \R^m)$ for a.e. $t\in[0,1]$ such that 
the following relations are fulfilled:
\begin{itemize}
\item[i)] $(e_t)_\sharp\left((y_2,x_2)-(y_1,x_1)\right)\alpha)=v_t\mu_t$ for a.e. $t\in[0,1]$,
\item[ii)] $\|v_t\|_{L^2_{\mu_t}} \leq W_{2,Y}(\mu_0,\mu_1)$ for a.e. $t\in[0,1]$,
\item[iii)] for $j \leq d$ we have that  $v_j=0$ for all $(t,y,x)\in[0,1]\times\R^d\times \R^m
$,
\item[iv)] $\mu_t,v_t$ fulfill the continuity equation 
$$\partial_t \mu_t + \nabla\cdot(v_t\mu_t)=0$$
in a distributional sense, i.e. we have for all $\varphi\in C^\infty_c((0,1)\times\R^d \times \R^m)$ that
\begin{align}
\int_0^1\int_{\R^d \times \R^m} \frac{\partial}{\partial t}\varphi  + \langle \nabla_{y,x} \phi,v_t\rangle \, \d \mu_t\d t=0.
\end{align}
\end{itemize}
Here $C_c^{\infty}$ denotes the smooth functions with compact support.
\end{proposition}
The proof is given in Appendix \ref{space:properties} and parts $i),ii),iv)$ are adapted from the proofs of \citep[Theorem 17.2, Lemma 17.3.]{ambrosio2021lectures}

Furthermore, since by Lemma \ref{plan:geo_dis} iii), a geodesic induced by an optimal $W_{2,Y}$ plan is weakly continuous, we obtain the following proposition 
from \citep[Proposition 8.1.8]{ambrosio2005gradient} 
which gives a connection to a flow equation and is needed for flow matching.

\begin{proposition}
Let $\mu_0,\mu_1\in \P_{2,Y}(\R^d\times \R^m)$. 
Let $\alpha\in \Gamma_{Y}^4(\mu_0,\mu_1)$ be an optimal plan in \eqref{definition}
and $\mu_t = (e_t)_\sharp \alpha$, $t \in [0,1]$.
Assume that  the corresponding Borel vector field $v_t$ from Proposition \ref{prop:continuity} fulfills 
\begin{align} \label{star}
    \int_0^1\left(  \sup_B(\|v_t\|_{ L^2_{\mu_t} })+\mathrm{Lip}(v_t,B) \right)\d t<\infty
\end{align}
for all compact subsets $B\subset \R^d \times \R^m$, where $\mathrm{Lip}(v_t,B)$ denotes the Lipschitz constant of $v_t$ on $B$. Then, for $\mu_0$-a.e. $(y,x)\in\R^d \times \R^m$, the ODE 
\begin{align}
&\frac{\d}{\d t}\phi_t=v_t(\phi_t),\\
& \phi_0(y,x)= (y,x),
\end{align}
admits a unique global solution and $\mu_t= ( \phi_t )_\sharp \mu_0$, $t \in [0,1]$.
\end{proposition}

For special cases we can drop the requirement \eqref{star} on the Borel vector field as the following proposition, which is proved in the Appendix \ref{space:properties} , shows.

\begin{proposition}\label{prop:flow_exists}
For $y_i \in \R^d$, $i=1,\ldots,n$, let $P_Y \coloneqq \frac 1 n\sum_{i=1}^n\delta_{y_i}$.
Let $\mu_0,\mu_1\in \P_{2,Y}(\R^d\times \R^m)$ fulfill one of the following conditions:
\begin{itemize}
    \item[i)] $\mu_{0,y},\mu_{1,y}$ are empirical measures with the same number of particles $N\in \N$ for $P_Y$ a.e. $y\in \R^d$. Let $T_{y_i}$ be a choice of optimal transport maps between $\mu_{0,y_i}$ and $\mu_{1,y_i}$ and let $\alpha$ be the corresponding optimal plan $\alpha\in \Gamma^4_{Y}(\mu_0,\mu_1)$, or
    \item[ii)] $\mu_{0,y},\mu_{1,y}$ for $P_Y$-a.e. $y\in\R^d$ 
    are absolutely continuous with densities $\rho_{0,y},\rho_{1,y}$ which are supported on open, convex, bounded, connected, subsets $\Omega_{0,y},\Omega_{1,y}$ on which they are bounded away from $0$ and $\infty$. 
    Assume further that $\rho_{0,y}\in C^2(\Omega_{0,y}),\rho_{1,y}\in C^2(\Omega_{1,y})$ and let $T_{y}$ be the associated optimal transport maps and 
    $\alpha\in \Gamma^4_{Y}(\mu_0,\mu_1)$ be the associated optimal transport plan.
    
\end{itemize}
Let $\mu_t = (e_t)_\sharp \alpha$ with associated vector field $v_t\in L^2_{\mu_t}$, where $(v_t)_j=0$ for all $j\leq d$. Then there is a representative of $v_t$ such that the flow equation
\begin{align}
&\frac{\d}{\d t}\phi_t=v_t(\phi_t)\\
&\phi_0(y,x)=(y,x)
\end{align}
admits a global solution and $\mu_t= (\phi_t)_\sharp \mu_0$. 
Furthermore, for $T \in L^2_{\mu_0}$ defined by $T(y_i,x) \coloneqq (y_i,T_{y_i}(x)) $, we have 
\begin{align}
v_t(\phi_t(y,x))=T(y,x)-(y,x)=\left(0,\pi^2\circ T(y,x)-x\right)
\end{align}
for $\mu_0$-a.e. $(y,x)\in \R^{d} \times \R^m$.
\end{proposition}
 The following proposition is a consequence of \cite[Corollary 1.2]{gonzalez2024linearization}. We only give an informal formulation here and refer for the details to Proposition \ref{prop:app_existence_ode} in the appendix. Notably, this helps us to overcome measurability issues when working with continuous $P_Y$ and therefore is applicable to a broader class of inverse problems.
\begin{proposition}\label{prop:cont_Monge}
Let $P_Y\in \P_2(\R^d)$, $\mu_0=P_Y\times \mu^Z_0$ and let $\mu_1=\mu_1^y\times_yP_Y$ with density $p_1^y$ of $\mu_1^y$. Assume further that the map $y\mapsto p_1^y$ is a $C^1$ map. Under  relatively mild assumptions, see Proposition \ref{prop:app_existence_ode}, the following statements hold true. There exists a $W_{2,Y}$-optimal transport map $T:(y,x)\mapsto (y,T_y(x))$ i.e. $\alpha=(\Id, T)_\sharp \mu_0\in \Gamma_{o,Y}(\mu_0,\mu_1)$, where $T_y$ is the optimal transport map for $\mu_0^Z$ and $\mu_1^y$.
Let $\mu_t \coloneqq (e_t)_\sharp \alpha$ with associated vector field $v_t\in L^2_{\mu_t}$, where $(v_t)_j=0$ for all $j\leq d$. Then there is a representative of $v_t$ such that the flow equation
\begin{align}
\frac{\d}{\d t}\phi_t=v_t(\phi_t); \quad\phi_0(y,x)=(y,x)
\end{align}
admits a global solution and $\mu_t= (\phi_t)_\sharp \mu_0$. 
Furthermore, we have 
\begin{align}
v_t(\phi_t(y,x))=T(y,x)-(y,x)=\left(0,T_y(x)-x\right)
\end{align}
for $\mu_0$-a.e. $(y,x)\in \R^{d} \times \R^m$.
\end{proposition}

\section{Relaxation of the Conditional Wasserstein Distance} \label{sec:relax}
When working with conditional Wasserstein distances, we are facing the following problems:

\begin{enumerate}
    \item We cannot use standard optimal transport algorithms like \citep{flamary2021pot} out of the box.
    \item Assume that $P_Y$ is not empirical and let $\mu\in \P_{2,Y}(\R^d \times \R^m)$. Then it is impossible to approximate $\mu$ by empirical measures in the $W_{2,Y}$ topology, since $\P_{2,Y}(\R^d \times \R^m)$ does not contain any empirical measures.
    \item Assume that we are interested in the optimal plan $\alpha\in\Gamma_Y^4(\mu_0,\mu_1)$, but we are only given empirical measures $\mu_{n,0},\mu_{n,1}$, which are $W_2$ approximations of $\mu_0,\mu_1$. Let $Y_n$ be a random variable distributed as $\pi^1_\sharp\mu_{n,0}$. Even if we assume that $\pi^1_\sharp\mu_{n,0}=\pi^1_\sharp\mu_{n,1}$,
    Example \ref{ex:emp_lim} shows that we cannot guarantee that there exists a sequence of the optimal plans $\alpha_n\in\Gamma_{Y_n}^4(\mu_{n,0},\mu_{n,1})$ that converges in any meaningful sense to $\alpha$.
\end{enumerate}

\begin{example}\label{ex:emp_lim}
We consider independent, standard normally distributed random variables $Y,X,Z \in \R$.
Let $\mu=\nu:=P_{Y,X}$. Now we sample $(y_i,x_i,z_i)\sim (Y,X,Z)$ 
and define 
$$\mu_n:= \frac{1}{n}\sum_{i=1}^n\delta_{y_i,x_i}, \quad \nu_n:=\frac{1}{n}\sum_{i=1}^n\delta_{y_i,z_i},$$ i.e. $\mu_n \to \mu$ and $\nu_n \to\nu$ as $n \to \infty$ in the $W_2$-topology. Note that we cannot compare $\mu_n$ and $\mu$ in the $W_{2,Y}$ topology since $\mu_n\notin \P_{2,Y}(\R\times\R)$.
Let $Y_n$ be a random variable distributed like $\frac 1 n\sum_{i=1}^n\delta_{y_i}$. 
Then with probability one $\Gamma_{Y_n}^4(\mu_{n},\nu_{n})$ 
contains exactly one plan 
$\alpha_n= \frac{1}{n}\sum_{i=1}^n\delta_{y_i,x_i,y_i,z_i} $
which is clearly optimal. 
Let $\Delta:\R^{3}\to \R^{4}$ be defined by $\Delta(y,x,z)=(y,x,y,z)$. Then $\hat{\alpha}:=\lim_{n}\alpha_n=\Delta_\sharp \left(P_{Y}\otimes P_X\otimes P_Z\right)$ in the $W_2$-topology.  Moreover, $\hat{\alpha}\in \Gamma^4_{Y}(\mu,\nu)$ and
\begin{align}
\int_{R^{4}}\|(y_1,x_1)-(y_2,x_2)\|^2\d \hat{\alpha}
&= \int_{\R^{3}}\|(y_1,x_1)-(y_1,x_2)\|^2 \, \d(P_{Y}\otimes P_X\otimes P_Z)\\
&=\int_{\R^{2}}\|x_1-x_2\|^2\, \d(P_{X}\otimes P_Z)>0 
= W_{2,Y}(\mu,\nu).
\end{align}

Hence $\hat{\alpha}$ is not an optimal coupling, although it is a limit of optimal couplings in the $W_2$ sense.
\end{example}

In order to overcome the above drawbacks, we define a cost function for 
which the OT plan $\alpha \in \Gamma(\mu_0,\mu_1)$ only approximately fulfills $\pi^{1,3}_{\sharp}\alpha=\Delta_{\sharp}P_Y$: 
$$
d_{\beta}^p ((y_1,x_1),(y_2,x_2)) = \Vert x_1 - x_2 \Vert^p + \beta \Vert y_1-y_2 \Vert^p, \quad
p \in [1,\infty), \; \beta >0.
$$ 
For large values of $\beta$, it  is very costly to move mass in $y$-direction.
Then we  denote by $W_{p,\beta}$ the OT distance with respect to the cost $d_\beta^p$, i.e. for $\mu_0,\mu_1\in \P_p(A\times B)$ we set
\[
W_{p,\beta}(\mu_0,\mu_1)^p
\coloneqq
\inf_{\alpha\in\Gamma(\mu_0,\mu_1)}\int_{(A\times B)^2} \d_\beta^p ((y_1,x_1),(y_2,x_2)) \, \d\alpha.
\]
The same idea has been pursued in \citep{alfonso2023generative}, where the authors rescaled the $y$-costs to obtain a block-triangular map in the Knothe-Rosenblatt setting \citep{knothe, rosenblatt} and similarly in \citep{hosseini2024conditional}. Note that \citep{hosseini2024conditional} was published on ArXiv after our first version \citep{CHW2023} of the present paper.  The distance $W_{2,\beta}$ was also discussed in \cite[Proposition 3.10]{peszek2023heterogeneous} and Proposition \ref{prop_beta} can be found in the proof thereof.

\begin{proposition}\label{prop_beta}
Let $\mu_0,\mu_1\in \P_{p,Y}(\R^{d} \times \R^m)$ and let $\alpha^\beta$ be a sequence of optimal transport plans for $\mu_0,\mu_1$ with respect to $W_{p,\beta}$.
Then, for $\beta\to\infty$, we have 
\[\int_{\R^{2d}} \|y_1-y_2\|^p \, \d{\pi^{1,3}}_{\#}(\alpha^{\beta}) \rightarrow 0.\]
\end{proposition}

\begin{remark} \label{y-scaling}
Alternatively, instead of rescaling the costs $d_{\beta}$ we would also rescale the inputs, which was done for instance in \citep{du2023nonparametric,hagemann2023posterior}. Take for instance (as we do numerically) the cost function $d_{\beta}^2 =  \Vert x_1 - x_2 \Vert^2 + \beta \Vert y_1-y_2 \Vert^2$. Then this is equivalent to rescaling the $Y$-component by $\sqrt{\beta}$.
\end{remark}

The following proposition shows that the issue raised in Example \ref{ex:emp_lim} is addressed by $W_{2,d_\beta}$.
\begin{proposition}
\label{prop:ex_fix}
Assume that $\mu,\nu\in P_{2,Y}(\R^d\times \R^m)$ and let $\mu_n,\nu_n$ be empirical measures that converge weakly to $\mu,\nu$. Then for a sequence $\beta_k\to\infty$ there exists an increasing subsequence $n_k$ and optimal plans $\alpha_{n_k}\in \Gamma(\mu_{n_k},\nu_{n_k})$ for $W_{2,d_{\beta_k}}(\mu_{n_k},\nu_{n_k})$ such that $\alpha_{n_k}$ converges weakly to an optimal plan $\alpha\in\Gamma^4_Y(\mu,\nu)$ for $W_{2,Y}(\mu, \nu)$.
\end{proposition}

\section{Bayesian Flow Matching} \label{sec:fm}
In this section, we combine the conditional Wasserstein distance with Bayesian flow matching. 
We start by briefly recalling flow matching and its OT variant. Then we turn to the
conditional setting, where we describe a method from the literature, which we call "random" Bayesian flow matching and introduce our new OT Bayesian flow matching.

\begin{remark}
Usually, in conditional generative modeling, the word "conditional" appears  the context of inverse problems (or solving class conditional problems). 
However, in the vanilla flow matching \citep{lipman2023flow} the word "conditional" is used for paths fixing a target sample and the whole procedure is referred to as "conditional flow matching". Therefore, we will call the flow matching procedure for inverse problems "Bayesian flow matching". 
\end{remark}

Flow matching and OT flow matching aim to sample from a target distribution $P_X$
by learning the velocity field $v_t$ of a flow ODE \citep{chen2018neural}
\begin{align} \label{velocity}
\frac{d}{dt} \phi_t(x) &= v_{t}(\phi_t(x)), \quad t \in [0,1],\\
\phi_0(x) &= x, 
\end{align}
which transports samples from an initial distribution $P_Z$
to those from $P_X$.  Once an approximate  velocity field $v^{\theta}_t$ is learned,
it can be replaced in \eqref{velocity} to get a
neural ODE \citep{chen2018neural}.

\textbf{Flow Matching}
 The flow matching framework \citep{lipman2023flow, liu2023flow}  learns 
  $v^{\theta}_{t}$
based on \emph{linear interpolation} 
between independent $Z$ and $X$, i.e. 
$$X_t = (1-t)\ Z + t\ X,$$ 
and in \citep[Theorem D.3]{liu2023flow} it was shown that a suitable vector field is
$$
v_t(x)=\E\left[\frac{d}{d t}X_t|X_t=x\right]=\E[X-Z|X_t=x]. 
$$
Consequently, an approximating velocity field $v^\theta_t$ can be learned by minimizing the loss function 
$$
L_{\text{FM}}(\theta) \coloneqq \mathbb{E}_{(z,x) \sim P_Z \otimes P_X, t \sim U([0,1])} 
\big[\Vert v^{\theta}_t(x_t) - (x - z) \Vert^2 \Big].
$$
The objective $L_{FM}$ can be also derived differently, with ideas inspired by the score matching framework \citep{hyvarinen2005estimation, vincent11}. Then, instead of regressing to the true velocity field at $x_t$, they show that regressing to it has the same gradients when one conditions at $x$ ("conditional" flow matching \citep[Theorem 2]{lipman2023flow}), which leads to the same loss formulation.

\textbf{OT Flow Matching} 
In contrast to the above linear interpolation, the authors in  \citep{tong2023improving, pooladian2023multisample}  suggested to use the 
$W_2(P_Z,P_X)$  coupling $\gamma$, respectively its Monge map $T$ and the corresponding 
\emph{McCann interpolation} \citep{mccann} 
$$X_t \coloneqq T_t(Z) = (1-t)Z + tT(Z)$$
which leads to
$$
T(Z)-Z=\frac{d}{d t} X_t=v_t(X_t).
$$
By Proposition \ref{prop:empi_smooth_flow},
the associated Borel vector field of the geodesic is given by 
$v_t=(T-\text{Id})\circ T_t^{-1}$. 
In a minibatch setting, this corresponds to sampling $(\zb z, \zb x)$ from $P_Z \otimes P_X$ and calculating the optimal plan $\gamma$ between 
$\frac{1}{I} \sum_{i=1}^I \delta_{z_i}$ 
and 
$\frac{1}{I} \sum_{i=1}^I \delta_{x_i}$, see 
\eqref{eq:assign},
which is supported on $(z_i,T(z_i))_{i=1}^I$. 
Consequently, the loss function becomes
 $$
 L_{\text{OT}}(\theta) \coloneqq \mathbb{E}_{(z,x) \sim \gamma, t \sim U([0,1])} \big[\Vert v^{\theta}_t(x_t) - (x-z) \Vert^2 \big],$$
where $x_t:= T_t(z,x)$.
\\[1ex]

Let us turn to the conditional setting.
In inverse problems, samples from the posterior measure are usually not available. In the conditional setting the corresponding flow ODE reads

\begin{align} \label{velocity-cond}
\frac{d}{dt} \phi_t(y,x) &= v_{t}(\phi_t(y,x)), \quad t \in [0,1],\\
\phi_0(y,x) &= (y,x). 
\end{align}

To this end, we pick the target measure as the joint distribution $P_{Y,X}$ and start from $P_{Y,Z}$. We do not want mass movement in $Y$-direction, as this would mean the measurement would change and we would not sample the posterior, which amounts to the $Y$-component of $v_t$ being (close to) zero, cf. Proposition \ref{prop:continuity}. 

\textbf{Random Bayesian Flow Matching}
In \citep{wildberger2023flow,zheng2023guided} flow matching is extended to the conditional setting. Given independent $Z$ and $(Y,X)$ we again consider the linear interpolation between $Z$ and $X$ given by 
$$X_t = (1-t)\ Z + t\ X.$$
Then $(Y,X_t)$ interpolates between $(Y,Z)$ and $(Y,X)$. Consequently 
$$(0,X-Z) = \frac{d}{d t}(Y,X_t). $$
 This yields the random Bayesian flow matching loss 
\begin{align}\label{eq:lyfm}
L_{Y, \text{FM}}(\theta) = \mathbb{E}_{(z,y,x) \sim  P_Z \otimes P_{Y,X}, t \sim U([0,1])}[\Vert v_t^{\theta}(y,x_t) - (x-z)\Vert^2].
\end{align}
Under the assumption $y_i \neq y_j$ for $i \neq j$ this matching coincides with the only admissible plan in the conditional Wasserstein distance.
In general however, according to Example \ref{ex:emp_lim}, this approach does not approximate OT plans as $X$ and $Z$ are essentially independent. Furthermore drawing a minibatch $\left( (z_i,y_i,x_i) \right)_{i=1}^I$ corresponds to a random coupling of the $z$ and $x$ for each class which explains the name random Bayesian flow matching.

\textbf{OT Bayesian Flow Matching}
  For $P_{Y,Z},P_{Y,X}$ as in Proposition \ref{prop:flow_exists} or Proposition \ref{prop:cont_Monge} there exists an optimal plan $\alpha \in \Gamma_Y^4(P_{Y,Z},P_{Y,X})$ and corresponding map $T$. Furthermore there exists a vector field $v_t\in L^2_{\mu_t}$ such that there exists a solution $\phi_t$ to the flow equation which satisfies
\begin{align*}
    v_t(\phi_t(y,x))=T(y,x)-(y,x) = \left(0, (\pi^2\circ T) (y,x)-x\right).
\end{align*}
Setting 
$$X_t \coloneqq T_t(Y,Z) = (1-t)Z + t(\pi^2\circ T)(Y,Z)$$
we have that $(Y,X_t)$ interpolates between $(Y,Z)$ and $(Y,X)$ and 
$$(0,(\pi^2\circ T)(Y,Z) - Z) = \frac{d}{d t}(Y,X_t)=v_t(Y,X_t).$$

Given a minibatch $(\zb z,\zb y,\zb x) = \left( (z_i,y_i,x_i) \right)_{i=1}^I$
sampled from the product distribution $P_Z \otimes P_{Y,X}$, we can calculate the optimal map $T$ and corresponding conditional coupling $\alpha$ between $(\zb y, \zb z)$ and $(\zb y, \zb x)$. The plan $\alpha$ by construction is only supported on $(y_i,z_i,y_i,(\pi^2\circ T)(y_i,z_i)))$. 
Now let $(y,z,y,x) \sim \alpha$, then we have
\begin{align*}
    v_t(y,x_t) = T(y,z) - (y,z)  = (y,x) -(y,z) =  (0,x-z)
\end{align*}
where $x_t:=T_t(y,z)$.
This gives rise to the following loss
\begin{align}\label{eq:lyot}
L_{Y,\text{OT}}(\theta) = \mathbb{E}_{((y,z,y,x) \sim \alpha, t \sim U([0,1])} [\Vert v^{\theta}_t(y,x_t) - (x-z) \Vert^2].
\end{align}

In practice, we use Proposition \ref{prop_beta} to approximate the optimal coupling $\alpha$. Therefore we allow small errors in the $Y$-component, in order to move more optimally in the $X$-direction, which is more in the spirit of Proposition \ref{prop:ex_fix}. Numerically, instead of taking the optimal transport plan with respect to the modified cost function, we rescale the $Y$-part, see Remark \ref{y-scaling}.

\section{Numerical Experiments} \label{sec:numerics}
In this section, we want to show cases in which it is beneficial to use the conditional Wasserstein distance. 
First, we verify that the convergence result for an increasing parameter $\beta$ given in Proposition \ref{prop_beta} for particle flows to MNIST \citep{deng2012mnist}. 
Then we  show the advantages of our Bayesian OT flow matching procedure 
on a Gausian mixture model (GMM) toy example and on CIFAR10 \citep{krizhevsky2009learning} class-conditional image generation.

\subsection{Particle Flow Convergence}
 In this example, we minimize 
 $W_{Y,d_\beta} (P_{Y,X},P_{Y,Z})$ for the empirical measures. 
 We consider the particle flow, i.e., the flow from $(Y,Z)$ to $(Y,X)$ for empirical distributions which minimizes the objective $\mathcal{D}\left(\frac{1}{n}\sum_{i=1}^n \delta_{y_i,x_i}, \frac{1}{n}\sum_{i=1}^n \delta_{y_i,z(t)_i}\right)$ for an appropriate distance $\mathcal D$, see e.g. \citep{altekrueger23a}. More concretely we follow a particle flow path, i.e., a curve starting with $z(0)_i \sim \mathcal{N}(0,I)$ which fulfills 
 $$\dot z(t) = - \eta \nabla_{z(t)} \mathcal{D}\left(\frac{1}{n}\sum_{i=1}^n \delta_{y_i,x_i}, \frac{1}{n}\sum_{i=1}^n \delta_{y_i,z(t)_i}\right), $$
 for an appropriate scaling $\eta$ and given joint samples $(y_i,x_i)_{i=1}^n$.
 We choose $\mathcal{D}$ as an approximation of $W_{2,d_\beta}$ by rescaling $Y$
 and using the Sinkorn divergence as the sample based distance measure  \citep{genevay18, feydy2019interpolating}, where the "blur" parameter is chosen so small that it is close to the Wasserstein distance. This way we can numerically verify the convergence in Proposition \ref{prop_beta}. Note that there are no neural networks involved in this example.
 
 We see in Fig. \ref{fig:four_images} that increasing $\beta$ indeed yields plans which transport no mass in $Y$-direction anymore, which has the consequence that the generated images fit the corresponding label. 
 It can be  seen that already for $\beta = 5$ each row only has one type of digit. 

\begin{figure}
    \centering
    \begin{minipage}{0.245\textwidth}
        \includegraphics[width= \linewidth]{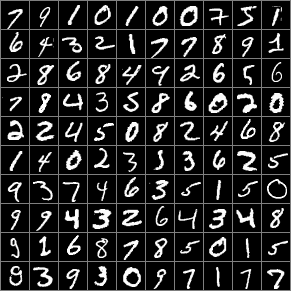} 
        \center{$\beta=1$}
    \end{minipage}%
    \hfill
    \begin{minipage}{0.245\textwidth}
        \includegraphics[width= \linewidth]{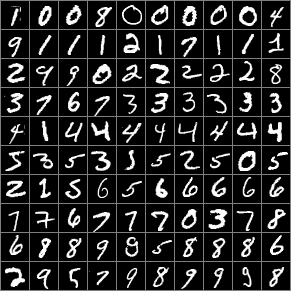} 
        \center{$\beta=3$}
    \end{minipage}%
    \hfill
    \begin{minipage}{0.245\textwidth}
        \includegraphics[width= \linewidth]{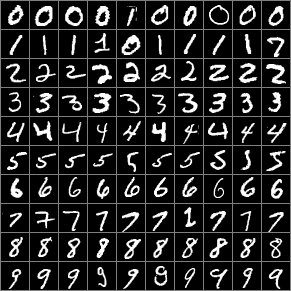}
        \center{$\beta=4$}
    \end{minipage}%
    \hfill
    \begin{minipage}{0.245\textwidth}
        \includegraphics[width= \linewidth]{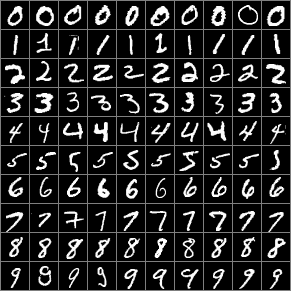} 
        \center{$\beta=5$}
    \end{minipage}%
    \caption{Class conditional MNIST particle flow  for different choices of $\beta$.
    With increasing $\beta$ the labels are better fitted.}
    \label{fig:four_images}
\end{figure}

\subsection{GMM Example}
Here we use an experimental setup from \citep{HHS22}:  Recall the Bayesian inverse problem setting in \eqref{inverse}, i.e. $Y = f(X) + \Xi$
where we choose $P_X$ to be a GMM in $\mathbb R^5$ with 10 mixture components, uniformly chosen means in $[-1,1]$ and standard deviation 0.1.
We apply a linear diagonal forward model $f = (f_{i,j})_{i,j=1}^5\in\R^{5\times 5}$  with $f_{i,i} = \frac{0.1}{i+1}$ and zero components otherwise
and choose $\Xi=\mathcal{N}(0,0.1)$ as a standard Gaussian distribution with standard deviation $0.1$. This yields a posterior measure $P_{X|Y=y}$ which is a also a GMM  \citep[Lemma 11]{HHS22}. Therefore we can sample and evaluate the true posterior as groundtruth. 
 We train a random Bayesian flow matching model according to $L_{Y,\text{FM}}$\eqref{eq:lyfm} and our OT Bayesian flow matching according to $L_{Y,\text{OT}}$\eqref{eq:lyot} with the python package POT \citep{flamary2021pot} on a fixed dataset of size $10000$, where we choose the best model according to a validation set of size $2000$ until the validation loss converges. For both approaches we use the same feed-forward neural network which contains around 140k parameters. We evaluate them using the Sinkhorn distance with blur 0.05 \citep{genevay18} with the package GeomLoss \citep{feydy2019interpolating} averaged with 100 posteriors and over 10 training and test runs with randomly sampled mixtures. The sampling is done via an explicit Euler discretization of 10 steps.
Our proposed OT Bayesian flow matching model trained according to $L_{Y,\text{OT}}$ with $\beta = 20$ obtains an average Sinkhorn distance of \textbf{0.0225}, whereas the random Bayesian flow matching model obtains a value of \textbf{0.0247}. In Figure \ref{fig_mix} one can see that both models approximate the posterior very well. We repeated the same experiments but used only $3$ Euler steps for sampling. In this case we obtained an average Sinkhorn distance of \textbf{0.0760} for the OT Bayesian flow matching and a value of \textbf{0.1044} for the random Bayesian flow matching. This indicates that the paths learned by our method are more straight.

\begin{figure}[p]
\centering
    \includegraphics[width=0.48\textwidth]{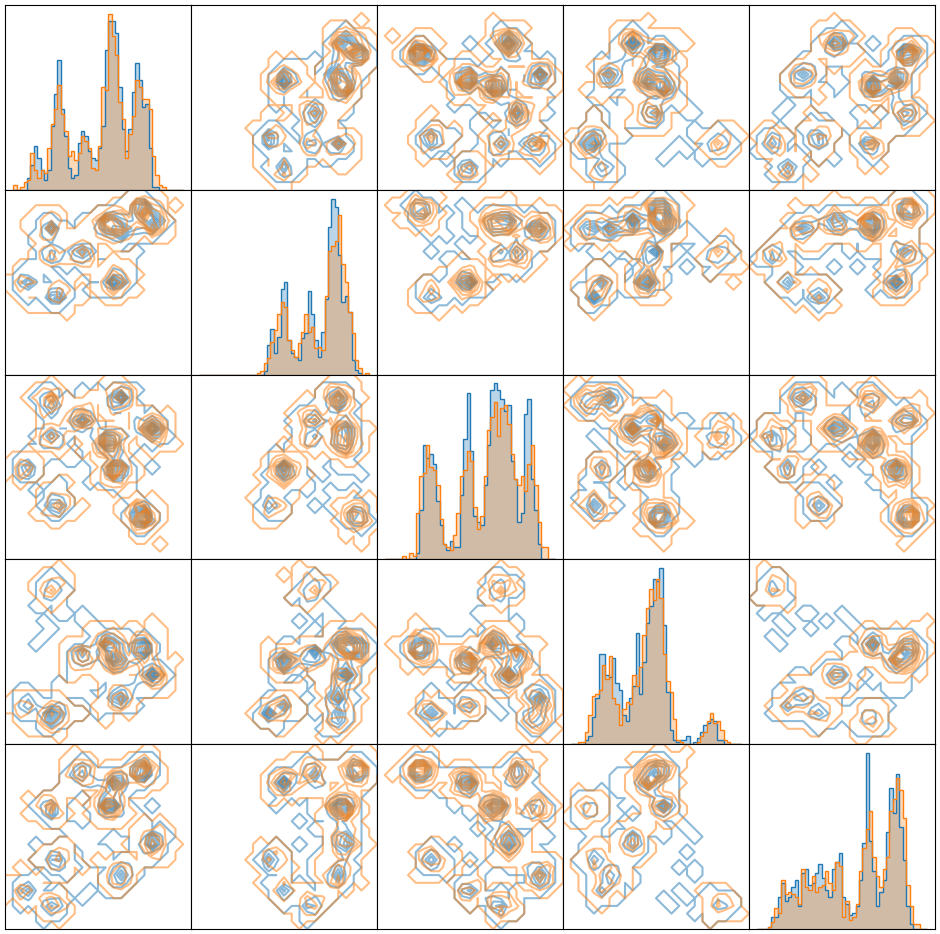}
    \includegraphics[width=0.48\textwidth]{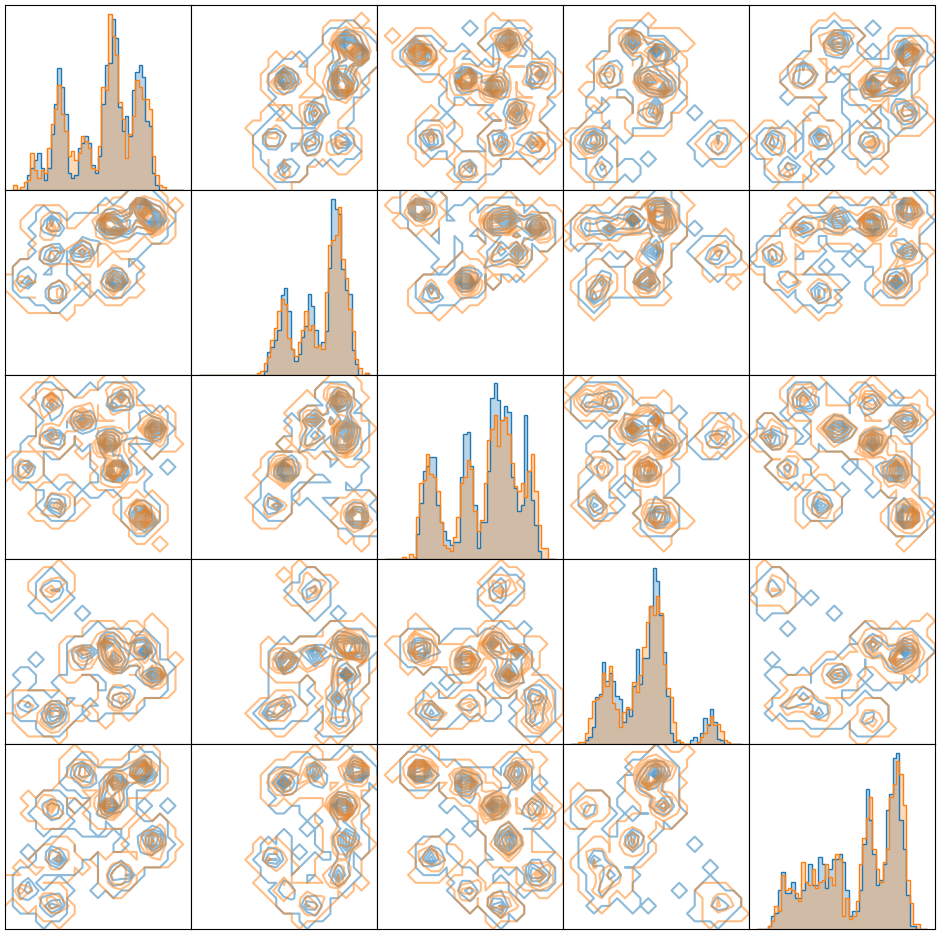}
    \includegraphics[width=0.48\textwidth]{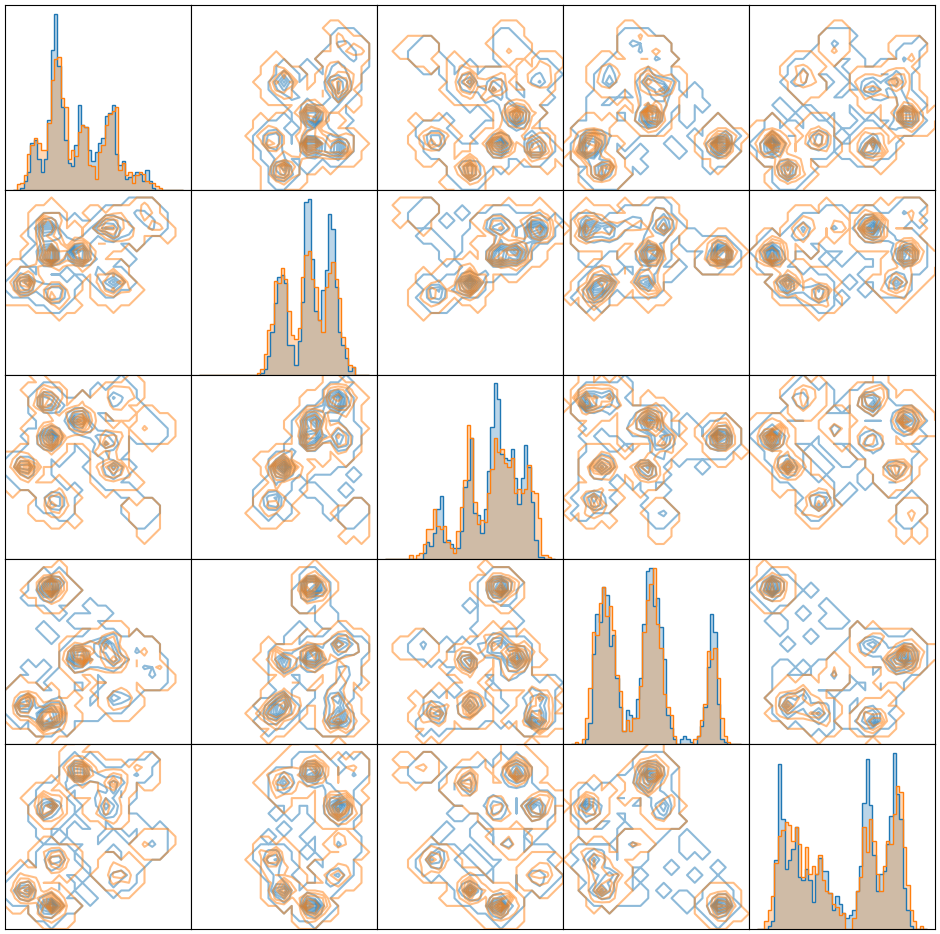}
    \includegraphics[width=0.48\textwidth]{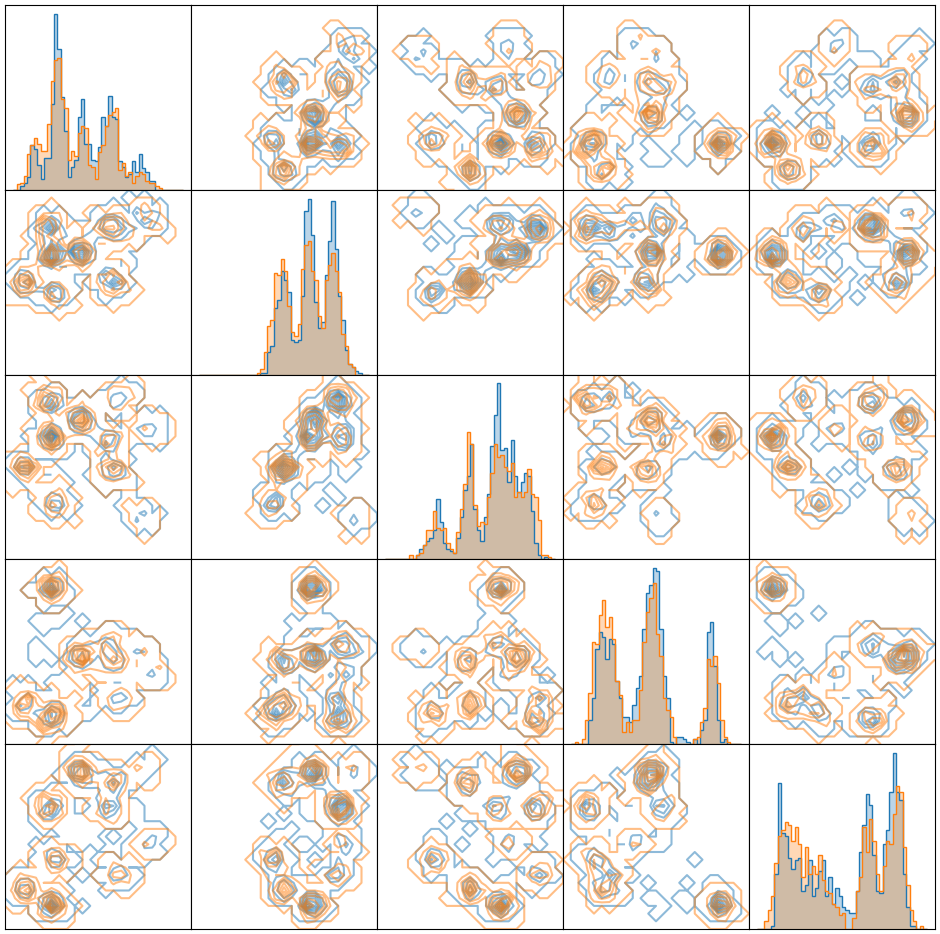}
    \caption{Posterior histograms for different methods with random Bayesian flow matching on the left and our OT Bayesian flow matching on the right for 10 Euler steps. Ground truth posterior is in orange and model prediction in blue. }
    \label{fig_mix}
\end{figure}

\subsection{Class Conditional Image Generation}
\begin{figure}[p]
    \centering
    \begin{minipage}{0.33\textwidth}
        \centering
        \includegraphics[width=\linewidth]{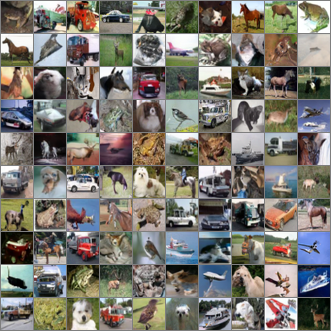} \\
        $\beta=1$
    \end{minipage}%
    \hfill
    \begin{minipage}{0.33\textwidth}
        \centering
        \includegraphics[width=\linewidth]{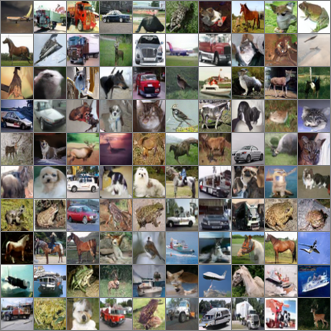} \\
        $\beta=3$
    \end{minipage}%
    \hfill
    \begin{minipage}{0.33\textwidth}
        \centering
        \includegraphics[width=\linewidth]{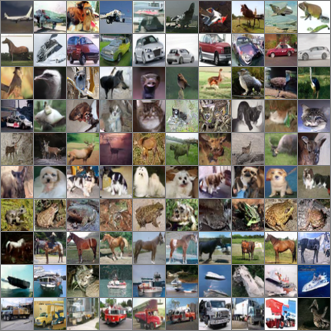} \\
        $\beta=5$
    \end{minipage}%
    \hfill
    \begin{minipage}{0.33\textwidth}
        \centering
        \includegraphics[width=\linewidth]{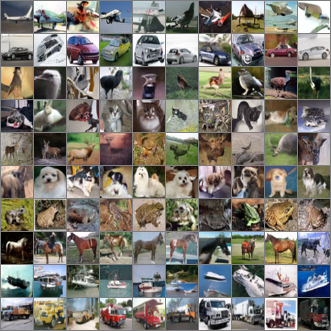} \\        
        $\beta=20$
    \end{minipage}%
    \hfill
    \begin{minipage}{0.33\textwidth}
        \centering
        \includegraphics[width=\linewidth]{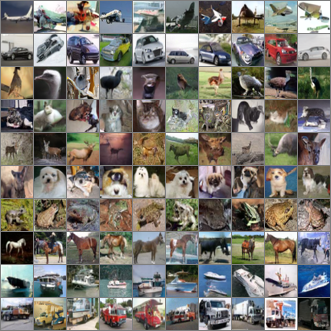} \\
        $\beta=100$
    \end{minipage}%
    \hfill
    \begin{minipage}{0.33\textwidth}
        \centering
        \includegraphics[width=\linewidth]{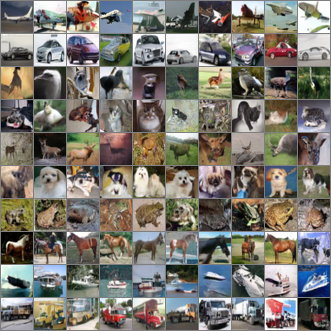} \\
        diagonal
    \end{minipage}%
    \hfill 
    \vspace{15pt}
   
    \begin{minipage}{1\textwidth}
    \centering
    \scalebox{.9}{
    
    \begin{tabular}[t]{c|cccccc}
    & 1 & 3 & 5 & 20 & 100 & diagonal \\
    \hline
    FID & 4.97 & 4.94 & 4.38 & 4.33 & \textbf{4.10} & 4.92 \\
    Epochs & 380 & 360 & 400 & 340 & 380 & 420 \\
    \end{tabular}}
    \end{minipage} 
    \caption{Class Conditional CIFAR results for different choices of $\beta$ and for the diagonal couplings. Additionally FID results are reported using an adaptive step size solver.}
    \label{fig:CIFAR}
\end{figure}

We apply our Bayesian OT flow matching for conditional image generation. We choose the condition $Y$ to be the class labels in order to generate samples of CIFAR10 for a given class. 
We simulate the flows for different values of $\beta$, by which we mean that we rescale the $Y$-component by $\beta$ as mentioned in Remark \ref{y-scaling}. We also simulate flows using the "diagonal" plans which coincide with the diagonal Bayesian flow matching objective \cite {wildberger2023flow}. For inference we simulate the flow ODE \cite{torchdiffeq} using an adaptive step size solver (Runge-Kutta of order 5) . The samples in Fig. \ref{fig:CIFAR} are generated using the adaptive step size solver and sorted by class labels. 
For low values of $\beta$ we see that the resulting samples do not match their class labels, increasing $\beta$ leads to accurate class representations. The samples are generated given the labels of the training samples, therefore we see improved FID results as $\beta$ increases. The diagonal flow matching objective has the correct class representations, however since the associated couplings are not optimal our experiments suggest that this leads to higher variance during training and therefore slightly lower image quality, see \cite{tong2023improving} for more details on the advantages of OT based flow matching. We run each method for 500 epochs and compute an FID over the validation set every 20 epochs. Then for each method we choose the best checkpoint and report the results. The code is written in PyTorch \cite{PyTorch2019} and is available online\footnote{\url{https://github.com/JChemseddine/Conditional_Wasserstein_Distances}}.

\section{Conclusions} \label{sec:conclusions}
Inspired from applications in Bayesian inverse problems, we introduced  conditional Wasserstein distances. We managed to rewrite these distances as  expectations of the Wasserstein distances with respect to the observation. Therefore we are able to directly infer posterior guarantees in expectation when trained with the corresponding losses. Furthermore, we calculated the dual of the conditional  Wasserstein-1 distance, when the probability measures are compactly supported and recovered well-known conditional Wasserstein GAN losses. 
We established  corresponding velocity fields for geodesics and used our results to design a new Bayesian flow matching algorithm. Moreover we use an approximation for the conditional Wasserstein distance depending on a parameter $\beta$ and show numerically that, when using it for class conditional  flow matching, the result does respect the classes for sufficiently large $\beta$. We achieve better FID than random Bayesian flow matching on Cifar10. Future work includes conditional domain translation, i.e., when the latent distribution is not a standard Gaussian, but given by some data distribution. There, finding a good OT matching and making use of our proposed framework could improve existing algorithms. 

\acks{Many thanks to F. Vialard for pointing to reference
\cite{peszek2023heterogeneous} after reading our arXiv version.
P. Hagemann acknowledges funding from from the DFG within the project SPP 2298 "Theoretical Foundations of Deep Learning",
C. Wald and G. Steidl gratefully acknowledge funding by the DFG within the SFB “Tomography Across the Scales” (STE 571/19-1, project number: 495365311) and by the DFG within the Exellencecluster MATH+.}

 \appendix
 
\section{Counterexample for Equality in Equation \ref{w1_coupling}} \label{sec:ex}

\begin{figure}
\centering
\scalebox{0.6}{
\begin{tikzpicture}
    \fill [blue] (-0.1,-0.1) rectangle (0.1,0.1);
    \node[below left=14pt of {(0,0)}, outer sep=2pt,fill=white] {$\delta_{y_1,x_1}$};
    \fill [red] (0,5) circle [radius=4pt];
    \node[above left=14pt of {(0,5)}, outer sep=2pt,fill=white] {$\delta_{y_1,z_1}$};
    \fill [red] (2,0) circle [radius=4pt];
    \node[below right=14pt of {(2,0)}, outer sep=2pt,fill=white] {$\delta_{y_2,z_2}$};
    \fill [blue] (2-0.1,5-0.1) rectangle (2+0.1,5+0.1);
    \node[above right=14pt of {(2,5)}, outer sep=2pt,fill=white] {$\delta_{y_2,x_2}$};
    \draw [black] (0,0) -- (2,0);
    \draw [black] (0,5) -- (2,5);

    \draw [dotted, black] (0,0) -- (0,5);
    \draw [dotted, black] (2,0) -- (2,5);
\end{tikzpicture}
} \caption{Visualization of the example. 
The blue dots belong to $P_{Y,X}$ and the red dots to $P_{Y,Z}$.
The optimal coupling is the solid line, while
the optimal conditional coupling is the dotted one.}
    \label{fig:enter-label}
\end{figure}
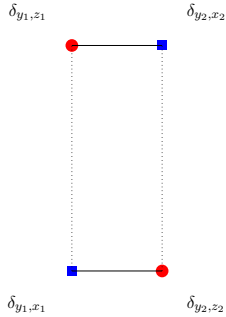

We provide a simple example showing that we cannot expect equality 
in \eqref{w1_coupling}.
Recall that for two empirical measures 
$\mu=\frac 1 n\sum_{i=1}^n\delta_{a_i}$ and 
$\nu=\frac 1 n \sum_{i=1}^n\delta_{b_i}$, $a_i,b_i \in \mathbb R^d$,
the Wasserstein-$p$ distance, $p \in [1,\infty)$ can  be written as
\begin{align}\label{eq:assign}
W_p^p(\mu,\nu) = \inf_{\sigma\in \mathcal S_n}\frac 1 n \sum_{i=1}^n\Vert a_i-b_{\sigma(i)}\Vert^p,
\end{align}
where $\mathcal S_n$ is the set of permutations on $\{1,\ldots,n\}$, see \citep[Proposition 2.1]{MAL-073}.

On the probability space $(\Omega, \mathcal{A}, \mathbb P)$ 
with $\Omega = \{\omega_1, \omega_2\}$, $\mathcal{A} = 2^{\Omega}$ and  $\mathbb P(\omega_1) = \mathbb P(\omega_2) = \frac{1}{2}$, we define the random variables 
$X,Y: \Omega \to \mathbb R$ by
\begin{center}
\begin{tabular}{c| c c c}
               &$X$&$Y$&$Z$\\
               \hline
    $\omega_1$&0&0&$n$\\
    $\omega_2$&$n$&1&0 
\end{tabular}\, , \hspace{0.5cm} $n >1$.
\end{center}
Then we have 
\begin{align*}
P_{Y,X} = \frac 1 2\delta_{0,0} + \frac 1 2 \delta_{1,n}, \quad  P_{Y,Z}= \frac 1 2 \delta_{1,0} + \frac 1 2 \delta_{0,n}
\end{align*}
which implies by \eqref{eq:assign} that
\begin{align*}
W_1(P_{Y,X},P_{Y,Z})  
&= \frac{1}{2}\min\big\{\|(0,0)-(1,0)\| + \|(1,n)-(0,n)\|,
\\ &\qquad\qquad 
\|(0,0)-(0,n)\|+\|(1,n)-(1,0)\|\big\} = 1.
\end{align*}
On the other hand, we get
$$
P_{X|Y=0}= \delta_{0}, \quad
P_{X|Y=1}=\delta_{n}, \quad
P_{Z|Y=0} = \delta_{n}, \quad
P_{Z|Y=1} = \delta_0, \quad
P_{Y} = \frac 1 2 \delta_0 + \frac 1 2 \delta_1,
$$
so that
\begin{align*}
\mathbb{E}_y[W_1(P_{X|Y=y},P_{Z|Y=y})] =  n = n W_1(P_{Y,X},P_{Y,Z}).
\end{align*}
Note that if we forbid the coupling to move mass across the $y$-direction, we actually would obtain equality, which motivates our definition of conditional Wasserstein distance, for an illustration see Fig. \ref{fig:enter-label}.

Note that in \citep{CondWasGen}, the summation metric is considered, i.e. 
$\Vert (x_1,y_1)- (x_2,y_2) \Vert_{sum} = \Vert x_1-x_2 \Vert + \Vert y_1-y_2 \Vert$ for which our counterexample is still valid.

\section{Proofs of Section \ref{sec:condW}} \label{app:condw}
\textit{Proof of Proposition \ref{cond:plan}.}
 iii)          
         Let $\alpha$ be defined by \eqref{def_alpha} which was already used in the proof of \citep[Theorem 2]{CondWasGen} , i.e.,
         \begin{align*}
         &\int\limits_{(A\times B)^2} f(y_1,x_1,y_2,x_2) \, \d \alpha(y_1,x_1,y_2,x_2) \\
 &= 
 \int\limits_A\int\limits_{A\times B^2}f(y_1,x_1,y_2,x_2) \, \d (\delta_{y_1}\times \alpha_{y_1})(y_2,x_1,x_2)\d P_Y(y_1)
 \end{align*}
 for all Borel measurable functions $f: (A \times B)^2 \to [0,+\infty]$.
Indeed $\alpha$ is a well defined probability measure on $(A\times B)^2$ by the following reasons:     
         by \citep[Lemma 12.4.7]{ambrosio2005gradient}, we can choose a Borel family $(\alpha_y)_y$.
         For any Borel set $\mathcal S \subseteq A\times B\times B$, we have 
         \begin{align}
         (\delta_{y}\times\alpha_y)(\mathcal S )=
         \int_{A\times B^2} 1_{\mathcal{S}}(\tilde{y},x_1,x_2) \, \d (\delta_y\times\alpha_y)(\tilde{y},x_1,x_2)
         =\int_{B^2}1_{\mathcal{S}}(y,x_1,x_2)\, \d \alpha_y.
         \end{align}
         By \citep[Equation 5.3.1]{ambrosio2005gradient} the function $y\mapsto \int_{B^2}1_{\mathcal{S}}(y,x_1,x_2)\d \alpha_y$ is Borel measurable. Consequently also $y\mapsto \delta_y\times\alpha_y(\mathcal S)$ is Borel measurable and thus $\alpha$ is well defined.
                
        It remains to show that $\alpha \in\Gamma_Y(P_{Y,X},P_{Y,Z})$ 
        which means $\pi^{1,3}_{\sharp}\alpha= \Delta_{\sharp}P_Y$, 
        $\pi^{1,2}_{\sharp}\alpha=P_{Y,X}$ and 
        $\pi^{3,4}_{\sharp}\alpha = P_{Y,Z}$. 
        The first equality follows from
        \begin{align*}
        \int_{A^2}f(y_1,y_2)\d \pi^{1,3}_{\sharp}\alpha 
        &= \int_{(A\times B)^2} (f\circ \pi^{1,3})(y_1,x_1,y_2,x_2) \, \d\alpha(y_1,x_1,y_2,x_2) \\
                &= \int_{(A\times B)^2}f(y_1,y_2) \, \d \delta_{y_1}(y_2)\, \d \alpha_{y_1}(x_1,x_2)\d P_Y(y_1)\\
                &= \int_{A}f(y,y) \, \d P_Y(y) 
               = \int_{A^2}f(y_1,y_2) \, \d(\Delta_{\sharp}P_Y)(y_1,y_2)
        \end{align*}
       for all Borel functions $f: A^2 \to [0,+\infty]$, and the second one from
        \begin{align}
            \int_{A\times B}f(y,x)\d\pi^{1,2}_{\sharp}\alpha(y,x) 
            &= \int_{(A\times B)^2}f(y_1,x_1) \, \d \delta_{y_1}(y_2)  \d \alpha_{y_1}(x_1,x_2)\d P_Y(y_1)\\
            &=\int_{A\times B}f(y,x)\d\pi^1_{\sharp}\alpha_y(x)\, \d P_Y(y)\\
            &=\int_{A\times B}f(y,x)\d P_{X|Y=y}(x)\, \d P_Y(y)\\
            &=\int_{A \times B} f(y,x) \, \d P_{Y,X}(y,x)
        \end{align}
        for all Borel functions $f: A\times B \to [0,+\infty]$. The third equality follows analogously. The optimality of $\alpha$ for $W_{p,Y}(P_{Y,X},P_{Y,Z})$ follows from \eqref{eq} which we show below.\\[1ex]
i) First we show $\geq$. 
        Let $\alpha_{y_1,y_2}$ be the disintegration of some 
        $\alpha\in\Gamma_Y^4(P_{Y,X},P_{Y,Z})$ with respect to $\pi^{1,3}_{\sharp}\alpha$.
        Then we obtain
        \begin{align}
            I(\alpha)  &:= \int_{(A\times B)^2}\|(y_1,x_1)-(y_2,x_2)\|^p \, \d \alpha(y_1,x_1,y_2,x_2)  \\
                &= \int_{A^2} \int_{B^2} \|(y_1,x_1)-(y_2,x_2)\|^p \, \d \alpha_{y_1,y_2}(x_1,x_2) \d \pi^{1,3}_{\sharp}\alpha(y_1,y_2)\\
                &= \int_{A^2} \int_{B^2} \|(y_1,x_1)-(y_2,x_2)\|^p \, \d \alpha_{y_1,y_2}(x_1,x_2)\, \d \Delta_{\sharp}P_Y (y_1,y_2)\\
                &= \int_{A} \int_{B^2}\|(y,x_1)-(y,x_2)\|^p \, \d\alpha_{y,y}(x_1,x_2)\d P_Y(y)\\
                &= \int_{A} \int_{B^2}\|x_1-x_2\|^p \, \d\alpha_{y,y}(x_1,x_2)\d P_Y(y).\label{mod:was}
        \end{align}
     Next, we show that $\alpha_{y,y} \in \Gamma(P_{X|Y=y},P_{Z|Y=y})$ a.e., which means 
     $\pi^1_{\sharp}\alpha_{y,y}=P_{X|Y=y}$ and $\pi^2_{\sharp}\alpha_{y,y}=P_{Z|Y=y}$ a.e.. 
     Using \eqref{disint_1}, we obtain indeed
            for all Borel measurable functions $f:A\times B\to [0,\infty]$ that
        \begin{align*}
                \int_{A} \int_{B} f(y,x_1) \, \d \pi^1_{\sharp}(\alpha_{y,y})(x_1)  \d P_Y(y) 
                &=
                \int_{A^2} \int_{B} f(y_1,x_1) \, \d\pi^1_{\sharp} \alpha_{y_1,y_2}(x_1)\, \d (\Delta)_{\sharp} P_Y(y_1,y_2)
                \\
                &=\int_{A^2} \int_{B} f(y_1,x_1) \, \d \pi^1_{\sharp}(\alpha_{y_1,y_2})(x_1) \, \d \pi^{1,3}_{\sharp}\alpha(y_1,y_2)\\
                &=\int_{A^2\times B^2}f(y_1,x_1) \, \d\alpha_{y_1,y_2}(x_1,x_2) \, \d \pi^{1,3}_{\sharp}\alpha(y_1,y_2)
                \\
                &= \int_{A^2 \times B^2}f(y_1,x_1) \, \d\alpha
                = \int_{A\times B} f(y_1,x_1) \, \d\pi^{1,2}_{\sharp}\alpha(y_1,x_1)\\
                &= \int_{A\times B} f(y_1,x_1) \, \d P_{Y,X}(y_1,x_1).
        \end{align*}    
     Consequently, we have $I(\alpha) \ge \mathbb E_{y\sim P_Y}\big[ W_p^p(P_{X|Y=y},P_{Z|Y=y})\big]$ 
     and since $ W_{p,Y}^p(P_{Y,X},P_{Y,Z}) = \inf_\alpha I(\alpha)$, this gives the assertion.
     
    Now we prove the opposite direction $\leq$. Let $\alpha \coloneqq \int_{A} \d \delta_{y_1}(y_2) \, \d \alpha_{y_1}(x_1,x_2) \d P_Y(y_1)$ be as in $iii)$ i.e. $W_p^p(P_{X|Y=y},P_{Z|Y=y})=\int_{B^2}\|x_1-x_2\|^p\d \alpha_y$. Then
    \begin{align}
    I(\alpha)&=\int_{(A\times B)^2}\|(y_1,x_1)-(y_2,x_2)\|^p\d\delta_{y_1}(y_1)\d\alpha_{y_1}(x_1,x_2)\\
    &=\int_A\int_{B^2}\|x_1-x_2\|^p\d \alpha_y\d P_Y=\mathbb{E}_{y\sim P_Y}\big[W_p^p(P_{X|Y=y},P_{Z|Y=y})\big]
    \end{align}
    which gives the assertion and we conclude $\mathbb{E}_{y\sim P_Y}\big[W_p^p(P_{X|Y=y},P_{Z|Y=y})\big]=W_{p,Y}(P_{Y,X},P_{Y,Z})$. We also obtain 
    \begin{equation} \label{eq}
         I(\alpha)\leq \mathbb{E}_{y\sim P_Y}\big[W_p^p(P_{X|Y=y},P_{Z|Y=y})\big]=W_{p,Y}(P_{Y,X},P_{Y,Z})
        \end{equation}
    which shows that the coupling $\alpha$ from $iii)$ is optimal for $W_{p,Y}(P_{Y,X},P_{Y,Z})$.\\[1ex]
ii) For an optimal $\alpha \in \Gamma_Y^4(P_{Y,X}, P_{Y,Z})$, we have by Part i) and \eqref{mod:was} that 
\begin{align}
W_{p,Y}^p (P_{Y,X}, P_{Y,Z} ) &= \int_A W_p^p(P_{X|Y=y}, P_{Z|Y=y}) \, \d P_y(y)\\
&= \int_A \int_{B^2} \|x_1-x_2\|^p \, \d \alpha_{y,y} (x_1,x_2) \d P_Y(y). 
\end{align}
Hence we get 
$$
0= \int_A \Big( \int_{B^2} \|x_1-x_2\|^p \, \d \alpha_{y,y} (x_1,x_2) -  W_p^p(P_{X|Y=y}, (P_{Z|Y=y}) \Big) \, \d P_Y(y).
$$
The inner integrand is nonnegative which finally implies
that it is zero $P_Y$-a.e. and therefore $\alpha_{y,y}$
is an optimal plan in $W_p(P_{X|Y=y}, P_{Z|Y=y})$.
\\[1ex]
 
\hfill $\blacksquare$
\\[2ex]
\textit{Proof of Proposition \ref{ew_1}.}
Let
$\kappa: A\times B^2 \to (A\times B)^2$ be defined by $(y,x_1,x_2) \mapsto (y,x_1,y,x_2)$.
 We  show that 
 $\kappa_{\sharp}: \Gamma_Y^3(P_{Y,X}, P_{Y,Z}) \to \Gamma^4_Y(P_{Y,X}, P_{Y,Z})$ 
 is the inverse of $\pi^{1,2,4}_{\sharp}$. 
 Since $\text{Id}_{(A\times B^2)} = \pi^{1,2,4}\circ(\Delta\circ\pi^1,\pi^2,\pi^3)$, it remains to show that
 $\kappa_{\sharp}\circ\pi^{1,2,4}_{\sharp}=\text{Id}_{\Gamma_Y^4(P_{Y,X},P_{Y,Z})}$.  For $\alpha \in \Gamma^4(P_{Y,X},P_{Y,Z})$ and Borel measurable function 
 $f: (A\times B)^2 \to [0,+\infty]$, we have
    \begin{align*}
    &\int_{(A\times B)^2}f(y_1,x_1,y_2,    x_2) \, \d \kappa_{\sharp}\pi^{1,2,4}_{\sharp}\alpha 
    = 
    \int_{(A\times B)^2}f(y_2,x_1,y_2,x_2)\, \d \alpha(y_1,x_1,y_2,x_2)\\
    &= 
    \int_{A^2} \int_{B^2}f(y_2,x_1,y_2,x_2)\, \d \alpha_{y_1,y_2}(x_1,x_2) \d \pi^{1,3}_{\sharp}\alpha(y_1,y_2)\\
    &= 
        \int_{(A \times B)^2}f(y_1,x_1,y_2,x_2) \, \d \alpha_{y_1,y_2}(x_1,x_2)\d \Delta_{\sharp}P_Y(y_1,y_2)\\
    &= 
    \int_{(A     \times B)^2} f(y_1,x_1,y_2,x_2) \, \d\alpha.
    \end{align*}
    
   The second claim follows
   by
   \begin{align*}
        &\int_{(A\times B)^2}\|(y_1,x_1)-(y_2,x_2)\|^p\, \d \alpha 
        =
        \int_{A^2} \int_{ B^2 }\|(y_1,x_1)-(y_2,x_2)\|^p\, \d \alpha_{y_1,y_2} \d \pi^{1,3}_{\sharp}\alpha(y_1,y_2)
        \\
        &=
        \int_{A^2} \int_{ B^2 }\|(y_1,x_1)-(y_2,x_2)\|^p\, \d \alpha_{y_1,y_2} \d \Delta_\sharp P_Y(y_1,y_2)\\
        &= \int_A\int_{B^2} \|(y,x_1)-(y,x_2)\|^p \d \alpha_{y,y} \d  P_Y (y)\\
        &= \int_A\int_{B^2} \|x_1-x_2\|^p \, \d \pi^{2,3,4}_{\sharp}\alpha .
        \qquad \blacksquare
        \end{align*}

\section{Proofs of Section \ref{sec:gan} 
 }\label{sec:dual}
The proof uses similar arguments as the short notes \citep{thickstun} and \citep{Basso2015AH}, which are derivations for the dual for the "usual" Wasserstein distance. We adapt these ideas for our conditional Wasserstein distance.
\\[2ex] 
\textit{ Proof of Proposition \ref{prop:gan}.}
Let $C_b = C_b(A\times B)$ be the space of continuous bounded functions on $A \times B$ and  $\mathcal S$ the set of nonnegative finite Borel measures $\alpha$ on $(A\times B)^2$ which are supported  at most on the y-diagonal.
By \citep[Section 1.2]{santambrogio2015optimal}, we know that
\begin{align*}
\sup_{f,g\in C_b(A\times B)}\E_{Y,X}[f] + \E_{Y,Z}[g] - \int_{(A \times B)^2} (f+g)\, \d\alpha
=\begin{cases}
    0 \quad\text{ if } \alpha\in\Gamma(P_{Y,X},P_{Y,Z}),\\
    \infty \quad\text{otherwise}. 
\end{cases}
\end{align*}
Using this relation, we obtain
\begin{align*}
 W_{1,Y}(P_{Y,X},P_{Y,Z}) &= \inf_{\alpha\in \Gamma_Y^4}\int_{}\|(y_1,x_1)-(y_2,x_2)\| \, \d\alpha \\
 &=\inf_{\alpha \in \mathcal S} \sup_{f,g \in C_b} L(\alpha, f,g)
 \end{align*}
with the Lagrangian
\begin{align} \label{lagrangian}
    L(\alpha,f,g)&\coloneqq\mathbb{E}_{Y,X}[f] + \mathbb{E}_{Y,Z}[g]  \\
    &\quad + \underset{(A\times B)^2}{\int} \Vert(y_1,x_1)-(y_2,x_2)\Vert -f(y_1,x_1) -g(y_2,x_2)  \, \d \alpha .
\end{align}
By Corollary \ref{thm:change} below, strong duality holds true, so that
we can exchange infimum and supremum to get  
$$
W_{1,Y}(P_{Y,X},P_{Y,Z}) = \sup_{f,g \in C_b} \inf_{\alpha \in \mathcal S}  L(\alpha, f,g).
$$
From this, we see that the optimal $f,g$ have to fulfill
\begin{align}\label{eq:leq}
f(y,x_1) + g(y,x_2) \leq \Vert x_1 - x_2 \Vert
\end{align}
for all $y\in A$, since otherwise the attained infimum is $-\infty$. 
Therefore we have for the optimal $f,g$ that 
$L(\alpha,f,g)\geq \mathbb{E}_{Y,X}[f] + \mathbb{E}_{Y,Z}[g] $ and choosing the plan $\alpha = 0 \in \mathcal S$, we obtain 
$$\inf_{\alpha\in \mathcal S} L(\alpha,f,g) = \mathbb{E}_{P_{Y,X}}[f] + \mathbb{E}_{P_{Y,Z}}[g]$$ for all $(f,g)\in \tilde{\mathcal F}$, where 
\[
\tilde{\mathcal F}:= \{(f,g) \in (C_b(A\times B))^2: f(y,x_1) + g(y,x_2) \leq \Vert x_2 - x_2 \Vert\}.
\]
Consequently, we get
\begin{align}\label{eq:wsup}
W_{1,Y}(P_{Y,X},P_{Y,Z}) = \sup_{(f,g) \in  \tilde{\mathcal F}} \mathbb{E}_{Y,X}[f] + \mathbb{E}_{Y,Z}[g].
\end{align}
For $(f,g)\in  \tilde{\mathcal F}$, we define 
$\tilde{f}(y,x) \coloneqq \inf_{u \in B} \Vert x-u \Vert - g(y,u)$. Then 
\begin{align}
    \tilde{f}(y,x) &= \inf_{u \in B} \{\|x-u\| - g(y,u)\} \\
    &\leq \inf_{u \in B} \{\|x-z\| + \|z-u\| - g(y,u)\} \\
    &= \tilde{f}(y,z) + \|x-z\|
\end{align}
    shows the 1-Lipschitz continuity of $\tilde{f}$ with respect to the second component. 
    Using \eqref{eq:leq} we obtain that $\tilde{f}(y,x) \geq f(y,x)$. 
    Since $\tilde{f}(y,x) \leq \|x-x\| - g(y,x)$, we conclude 
    \begin{align}\label{eq:ftild}
    f(y,x) \leq \tilde{f}(y,x) \leq -g(y,x).
    \end{align}
Thus, $\tilde{f}$ is bounded. As pointwise infimum over continuous functions, $\tilde{f}$ is upper semicontinuous in $(y,x)$. In summary, we have that $\tilde{f} \in \mathcal F$. 
By \eqref{eq:wsup} and  \eqref{eq:ftild}, we conclude 
$$
W_{1,Y}(P_{Y,X},P_{Y,Z})
= \sup_{(f,g) \in \tilde {\mathcal F}} \left\{ \mathbb{E}_{Y,X}[f] + \mathbb{E}_{Y,Z}[g] \right\}\\
\leq 
 \sup_{h \in \mathcal F} \left\{ \mathbb{E}_{Y,X}[h] - \mathbb{E}_{Y,Z}[h] \right\}
 $$
 and further for $\alpha \in \Gamma^4_Y(P_{Y,X}, P_{Y,Z}) \subset \Gamma(P_{Y,X}, P_{Y,Z})$ that
\begin{align*}
\sup_{h \in \mathcal F} \left\{ \mathbb{E}_{Y,X}[h] - \mathbb{E}_{Y,Z}[h] \right\}
&\leq \sup_{h \in \mathcal F} \inf_{\alpha \in \Gamma_Y^4}\int_{(A\times B)^2}h(y_1,x_1)-h(y_2,x_2) \d \alpha\\
&= \sup_{h\in \mathcal F}\inf_{\alpha \in \Gamma_Y^4}\int_{(A\times B)^2}h(y_1,x_1)-h(y_1,x_2)\d \alpha \\
&\leq \inf_{\alpha \in \Gamma_Y^4} \int_{(A \times B)^2} \Vert x_1 - x_2 \Vert \d\alpha \\
& = \inf_{\alpha \in \Gamma_Y^4}\int_{(A \times B)^2} \Vert (y_1,x_1)-(y_2,x_2)\Vert \d\alpha \\
&= W_{1,Y}(P_{Y,X},P_{Y,Z}).
\end{align*}
Thus, 
$W_{1,Y}(P_{Y,X},P_{Y,Z}) = \sup_{h\in \mathcal F} \{\mathbb{E}_{Y,X}[h] - \mathbb{E}_{Y,Z}[h]\}$, which finishes the proof.
\hfill $\blacksquare$
\\[1ex]

The proof of strong duality relies on the following minimax principle 
from \citep[Theorem 7 Chapter 6]{aubin2006applied}.

\begin{theorem} \label{minmax}
Let $X$ be a convex subset of a topological vector space, and $Y$ be a convex subset of a vector space. Assume $f: X \times Y \to \R$ satisfies the following conditions:
\begin{itemize}
    \item[i)] For every $y\in Y$, the map $x\mapsto f(x,y)$ is lower semi continuous and convex.
    \item[ii)] There exists $y_0$ such that $x\mapsto f(x,y_0)$ is inf-compact, i.e the set $\{x\in X: f(x,y_0) \leq a\}$ 
    is relatively compact for each $a \in \R$.
    \item[iii)] For every $x\in X$, the map $y \to f(x,y)$ is convex.
\end{itemize}
Then it holds
\begin{align*}
    \inf_{x\in X} \sup_{y\in Y} f(x,y) = \sup_{y\in Y} \inf_{x\in X} f(x,y).
\end{align*}
\end{theorem} 

Based on the theorem we can prove the desired strong duality relation.

\begin{corollary}\label{thm:change}
For the Lagrangian in \eqref{lagrangian}
it holds 
\begin{align*}
      \underset{\alpha\in \mathcal S}{\mathsf{inf}} \sup_{f,g \in C_b} L(\alpha,f,g) = \sup_{f,g \in C_b}\underset{\alpha\in \mathcal S}{\mathsf{inf}} L(\alpha,f,g) .
\end{align*}
\end{corollary}

\begin{proof} 
We will verify the conditions in Theorem \ref{minmax}.
Recall that $\mathcal S$ is the set of finite nonnegative Borel measures $\alpha$ on $(A\times B)^2$ such that there exists a finite nonegative finite measure $\beta$ on $B$ with $\pi^{1,3}_{\sharp}\alpha = \Delta_{\sharp} \beta$. Let $\mathcal{M}$ be the topological vector space of finite signed Borel measures on $(A\times B)^2$ with weak convergence topology. Thus, since the pushforward is linear on $\mathcal{S}$, 
we conclude that $\mathcal S$ is a convex subset. Now we use Theorem \ref{minmax} with $X \coloneqq \mathcal S$, $Y \coloneqq C_b\times C_b$
and $f \coloneqq L$.

\textit{Verifying i)} The map $\alpha \mapsto L(\alpha,f,g)$ is linear and continuous on $\mathcal S$  under the weak convergence of measures. This follows from the fact that the integrand of $\alpha$ in $L(\alpha, f,g)$ is in $C_b((A\times B)^2)$. 

\textit{Verifying iii)} Note that for any $\alpha \in \mathcal S$ the map $(f,g) \mapsto L(\alpha,f,g)$ 
is linear in $(f,g)$ and therefore convex.  

\textit{Verifying ii)} Setting $f(y,x) \coloneqq -1, g(y,x) \coloneqq -1$ for all $(y,x)$, we will show that for any fixed $a \in R$, the set 
\[\mathcal S_a:= \{\alpha \in  \mathcal S: L(\alpha,-1,-1) \leq a\}\]
is  relatively compact.
Since the integrand is bounded from below by $2$ and $\mathcal S$ only contains nonnegative measures, the measures in $\mathcal S_a$ are uniformly bounded in the total variation norm, since otherwise 
\begin{align*}
L(\alpha,-1,-1) = 2 + \int_{(A\times B)^2} \|(y_1,x_1)-(y_2,x_2)\| +2  \,  \d \alpha
    \end{align*}
can become arbitrary large  which contradicts  $L(\alpha,-1,-1) \leq a$.  
    Therefore the compactness of $A,B$ implies that  $\mathcal S_a$ is a family of tight measures. By \citep[Theorem 8.6.7]{Bogachev2007}, the set $\mathcal S_a$ is relatively compact in the weak topology.
\end{proof}

\section{Proofs of Section \ref{sec:velocity}}\label{space:properties}
\textit{Proof of Proposition \ref{plan:geo_dis}.}
i) We have $\mu_t\in\P_{p,Y}(\R^d\times \R^m)$ for every $t \in [0,1]$ by
\begin{align}
(\pi^1)_\sharp\mu_t
=\pi^1_\sharp (e_t)_\sharp\alpha
=
((1-t)\pi^{1} + t \pi^{3})_\sharp(\pi^{1, 3})_\sharp \alpha
= ((1-t)\pi^{1} + t \pi^{2})_\sharp \Delta_\sharp P_Y = P_Y.
\end{align}
For $s,t \in [0.1]$, let
$\alpha_{s,t}\coloneqq (e_s,e_t)_\sharp\alpha$.
By definition we see that $\alpha_{s,t} \in \Gamma(\mu_s,\mu_t)$.
Further $\pi^{1,3}_{\sharp}\alpha_{s,t}=\Delta_{\sharp}P_Y$ 
follows from 
\[
\pi^{1,3}\circ\left(e_s,e_t\right)
=
\left((1-s)\pi^{1}+s\pi^{2},(1-t)\pi^{1}+t\pi^{2}\right)\circ \pi^{1,3}
\]
and consequently
\begin{align}
    \pi^{1,3}_{\sharp}\alpha_{s,t}
    &=\left((1-s)\pi^{1}+s\pi^{2},(1-t)\pi^{1}+t\pi^{2}\right)_\sharp\pi^{1,3}_\sharp\alpha\\
    &=\left(\left((1-s)\pi^{1}+s\pi^{2},(1-t)\pi^{1}+t\pi^{2}\right)\circ \Delta\right)_\sharp P_Y
    = \Delta_\sharp P_Y.
\end{align} 
In summary, we see that $\alpha_{s,t} \in \Gamma_Y^4(\mu_s,\mu_t)$.
Thus, we have
\begin{align}
W_{2,Y}^2(\mu_s,\mu_t)
&\leq \int_{(\R^d \times \R^m)^2} \|(y_1,x_1)-(y_2,x_2)\|^2 \, \d\alpha_{s,t}\\
&=\int_{(\R^d \times \R^m)^2} \big \|(t-s)\left((x_1,y_1)-(x_2,y_2) \right) \big \|^2 \, \d\alpha\\
&= |t-s|^2 \, W_{2,Y}^2(\mu_0,\mu_1). \label{aax}
\end{align}
Finally, the desired equality follows like in \citep[Theorem 7.2.2]{ambrosio2005gradient}for $0\le s \le t \le 1$ by
$$
W_{2,Y}(\mu_0,\mu_1) \le W_{2,Y}(\mu_0,\mu_s)+ W_{2,Y}(\mu_s,\mu_t)+W_{2,Y}(\mu_t,\mu_1)
\le W_{2,Y}(\mu_0,\mu_1),
$$
which implies equality in \eqref{aax}.
\\
ii) First, we show $(e_t)_\sharp\alpha = (\pi^1,(1-t)\pi^2+t\pi^4)_\sharp\alpha$. 
For any Borel measurable  function $f: \R^d \times \R^m \to [0,\infty]$, we have indeed
\begin{align}
\int_{\R^d \times \R^m} f \, \d(e_t)_\sharp\alpha
&=\int_{(\R^{d} \times \R^m)^2} f((1-t)y_1+ty_2,(1-t)x_1+tx_2) \, \d\alpha\\
&=\int_{\R^{2d}}\int_{\R^{2m}}f((1-t)y_1+t y_2,(1-t)x_1+t x_2) \, \d\alpha_{y_1,y_2}\d \pi^{1,3}_\sharp\alpha\\
&=\int_{\R^{d}}\int_{\R^{2m}}f((1-t)y+t y,(1-t)x_1+t x_2) \, \d\alpha_{y,y}\d P_Y\\
&=\int_{\R^{d} \times \R^m}f \, \d (\pi^1,(1-t)\pi^2+t\pi^4)_\sharp\alpha.
\end{align}
Using the above relation, we obtain 
\begin{align}
&\int_{\R^d \times \R^m} f \, \d ((\mu_t)_y \otimes P_Y)
=\int_{\R^d} \int_{\R^m}f(y,x)\, \d ((1-t)\pi^1+t\pi^2)_\sharp\alpha_{y,y}(x) \d P_Y(y)
\\
&= \int_{\R^d}\int_{\R^{m\times m}}f(y, (1-t)x_1+t x_2) \, \d \alpha_{y,y}(x_1,x_2)\d P_Y(y)
\\
&= \int_{\R^{2d}} \int_{\R^{2m}} f(y_1, (1-t)x_1+t x_2) \, \d \alpha_{y_1,y_2}(x_1,x_2) \d \Delta_\sharp P_Y(y_1)\\
&= \int_{(\R^{d} \times \R^m)^2}
 f(y_1, (1-t)x_1+t x_2)\d \alpha(y_1,x_1,y_2,x_2)\\
&=\int_{(\R^{d} \times \R^m)^2} f \, \d(e_t)_{\sharp}\alpha
=\int_{\R^d \times \R^m}f \, \d\mu_t,
\end{align}
which proves that $(\mu_t)_y$ is indeed the disintegration of $\mu_t$
with respect to $P_Y$.
By Proposition \ref{cond:plan} ii) we know that
$\alpha_{y,y} \in \mathcal P(\R^{2m})$ is optimal in \eqref{eq:wasserstein}
for $P_y$-a.e. $y \in \R^d$.
By \eqref{eq:geodesic_plan} this implies that $(\mu_t)_y$ is a geodesic in $\mathcal P_2(\R^m)$. 
\\[1ex]
iii) 
Recall (see \citep[Section 5.1]{ambrosio2005gradient}) that a sequence $\mu_k\in \P(\R^n)$ is said to converge weakly to $\mu\in \P(\R^n)$ if 
$\lim_{k\to\infty}\int_{\R^n}f(x) \, \d\mu_k(x) = \int_{\R^n}f(x) \, \d\mu(x)$ 
for all $f \in C_b(\R^n)$.
By the dominated convergence theorem, we have for $\mu_s = (e_s)_\sharp \alpha$ and every 
$f \in C_b(\R^{d} \times \R^m)$ that
\begin{align}
\lim_{s\to t} \int_{\R^{d} \times \R^m} f \, \d \mu_s
&=\lim_{s\to t}\int_{(\R^{d} \times \R^m)^2} f((1-s)(y_1,x_1)-s(y_2,x_2)) \, \d \alpha\\
&=\int_{(\R^{d} \times \R^m)^2} f((1-t)(y_1,x_1)-t(y_2,x_2))\d \alpha
=\int_{\R^{d} \times \R^m} f \, \d\mu_t,
\end{align}
which finishes the proof.
 \hfill$\blacksquare$
\\[2ex]
\textit{Proof of Proposition \ref{prop:continuity}.}
The proof of $i),ii),iv)$ is almost identical to the proofs of \citep[Theorem 17.2, Lemma 17.3]{ambrosio2021lectures}. For the convenience of the reader and because we also need measurability of $v_t$ in $t$ we include a slight adaptation of their proof.
We let $e:[0,1]\times \R^{d}\times \R^m\times \R^d\times \R^m\to [0,1]\times\R^{d}\times \R^m$ be defined by $e(t,y_1,x_1,y_2,x_2)=(t,(1-t)(y_1,x_1)+t(y_2,x_2))$. For $\mathcal{L}$ the Lebesgue measure we have that 
\begin{align}
\int_{[0,1]\times\R^{d+m}}f(t,y,x)\d e_\sharp\left(\mathcal{L}\otimes \alpha\right)&= \int_{[0,1]}\int_{\R^{2d+2m}}f(t,e_{t}(y_1,x_1,y_2,x_2))\d\alpha\d t\\
&=\int_{[0,1]}\int_{\R^{d+m}}f(t,y,x)\d e_{t,\sharp}\alpha \d t=\int_{[0,1]}\int_{\R^{d+m}}f(t,y,x)\d \mu_t \d t
\end{align}
for every bounded $\mathcal{L}\otimes\alpha$ measurable function $f$. In particular using $f=1_{B}$ for a Borel set $B\in\mathcal{B}(\R^d\times \R^m)$ we see that $\mu_t:[0,1]\times \mathcal{B}\to \R$ is a Markov kernel, i.e. the measure $\int_{0}^1\mu_t\d t$ is well defined, and $e_\sharp \left(\mathcal{L}\otimes \alpha\right)=\int_0^1\mu_t\d t$.
Hence \citep[Lemma 17.3]{ambrosio2021lectures}
(with $e \coloneqq e$, $\mu = \mathcal{L}\otimes\alpha$, $v = (y_2,x_2) - (y_1,x_1)$, $w = v$) implies that there exists $v\in L^2_{\int\mu_t\d t}([0,1]\times \R^d\times \R^m)$ such that 
\begin{align}\label{eq:push=int}
e_\sharp\left(((y_2,x_2)-(y_1,x_1))\mathcal{L}\otimes\alpha\right)=v\int_0^1\mu_t\d_t
\end{align}
and we can choose a Borel measurable representative of $v.$ 
Since for test functions $f$ we have the following identities
\begin{align}
\int_{[0,1]\times \R^d\times\R^m}f\d e_\sharp\left(((y_2,x_2)-(y_1,x_1))\mathcal{L}\otimes\alpha\right)&=\int_0^1\int_{\R^d\times\R^m}f\d e_{t,\sharp}\left(((y_2,x_2)-(y_1,x_1))\alpha\right)\d t\\
\int_{[0,1]\times \R^d\times\R^m} f \d \left(v\int_0^1\mu_t\d_t\right)&=\int_0^1\int_{\R^d\times\R^m}f v_t\d \mu_t\d t
\end{align}
we obtain from \eqref{eq:push=int} that $e_{t,\sharp}\left(((y_2,x_2)-(y_1,x_1))\alpha\right)=v(t,y,x)\mu_t$ for $\mathcal{L}$ a.e. $t\in[0,1]$ which shows $i)$. By \citep[Lemma 17.3]{ambrosio2021lectures} we obtain that \[\|v_t(y,x)\|_{L^2_{\mu_t}}\leq \|((y_2,x_2)-(y_1,x_1))\|_{L^2_{\alpha}}=W_{2,Y}(\mu_0,\mu_1)\] for a.e. $t\in[0,1]$ and hence we can conclude $ii)$.

Towards iii), note that it holds for any Borel measurable set $U\subseteq (\R^d \times \R^m)^2$ and $j \leq d$ that 
\begin{align}
\left|\int_U (y_2)_j-(y_1)_j \, \d \alpha\right|
&\leq \int_U|(y_2)_j-(y_1)_j| \, \d \alpha 
\leq \int_{\pi^{1,3}(U)}|(y_2)_j-(y_1)_j| \, \d \pi^{1,3}_\sharp\alpha \\
&= \int_{\pi^{1,3}(U)}|(y_2)_j-(y_1)_j| \, \d \Delta_\sharp P_Y\\
&= \int_{\Delta^{-1}(\pi^{1,3}(U))}|y_j-y_j|\, \d P_Y=0.
\end{align}
Thus, for any Borel measurable set $V\subseteq [0,1]\times\R^d \times \R^m$ and $j \leq d$, we obtain by Part i) that
\begin{align}
\int_V v_j \, \d \mu_t\d t 
&= \int_V \, \d e_\sharp((y_2)_j- (y_1)_j)(\mathcal L\otimes\alpha))
=\int_0^1\int_{\tilde{V}}((y_2)_j- (y_1)_j) \, \d \alpha\d t=0
\end{align}
where $\tilde{V}=(\pi^{t})^{-1}(e^{-1}(V))$.
This implies $(v(t,y,x))_j=0$ for $\int_0^1\mu_t\d t$-a.e. $(t,y,x)\in [0,1]\times\R^d \times \R^m$ and we can choose a Borel measurable representative of $v$ such that $v_j=0$ for all $(t,y,x)\in [0,1]\times\R^d\times\R^m $\\[1ex]
Next we proof $iv)$. Let $\varphi\in C_c^{\infty}((0,1)\times\R^d\times\R^m)$. Then by the chain rule
\[
\frac{\partial}{\partial t}(\varphi(t,e_t))=\left(\frac{\partial}{\partial t}\varphi\right)\circ (t,e_t)+\langle \nabla_{y,x}\varphi(t,e_t),(y_2,x_2)-(y_1,x_1)\rangle.
\]
Consequently
\begin{align}
\int_0^1\int_{\R^{d+m}}\frac{\partial}{\partial t}\varphi\d\mu_t\d t &=\int_0^1\int_{\R^{2d+2m}}\left(\frac{\partial}{\partial t}\varphi\right)\circ (t,e_t)\d\alpha\d t\\
&=\int_0^1\int_{\R^{2d+2m}}\frac{\partial}{\partial t}(\varphi(t,e_t))-\langle \nabla_{y,x}\varphi(t,e_t),(y_2,x_2)-(y_1,x_1)\rangle\d\alpha\d t\\
&=\int_0^1\frac{\partial}{\partial t}\int_{\R^{2d+2m}}\varphi(t,e_t))\d \alpha\d t\\
&-\int_0^1\int_{\R^{2d+2m}}\langle \nabla_{y,x}\varphi(t,e_t),(y_2,x_2)-(y_1,x_1)\rangle\d\alpha\d t\\
&=0-\int_0^1\int_{\R^{d+m}}\langle \nabla_{y,x}\varphi,e_{t,\sharp}(y_2,x_2)-(y_1,x_1)\d\alpha\rangle\d t\\
&=-\int_0^1\int_{\R^{d+m}}\langle \nabla_{y,x}\varphi,v_t\rangle\d\mu_t\d t
\end{align}

where we used 
\[
\int_0^1\frac{\partial}{\partial t}\int_{\R^{2d+2m}}\varphi(t,e_t))\d \alpha\d t=\int_{\R^{2d+2m}}\varphi(1,e_1))\d \alpha-\int_{\R^{2d+2m}}\varphi(0,e_0))\d \alpha=0
\]
since $\varphi(1,\cdot)=\varphi(0,\cdot)=0$ because $\varphi$ is compactly supported on $(0,1)\times \R^{d}\times \R^m$.
\hfill $\blacksquare$
\\[2ex]

For the proof of  Proposition \ref{prop:flow_exists} we need the following proposition. Since we have not found a proof in the literature, we give it for convenience.

\begin{proposition}\label{prop:empi_smooth_flow}
Let $\mu_0, \mu_1\in (\P_2(\R^m), W_2)$ which fulfill one of the following conditions:
\begin{itemize}
    \item[i)] $\mu_0,\mu_1$ are empirical measures with the same number of points and $T$ is an optimal map with associated optimal plan $\alpha\in\Gamma(\mu_0,\mu_1)$, or
    \item[ii)] $\mu_0,\mu_1$ both admit densities $\rho_0,\rho_1$ which are supported on open, convex, bounded, connected subsets $\Omega_0,\Omega_1 \subset \R^m$ on which they are bounded away from $0$ and $\infty$. 
    Assume further that $\rho_0 \in C^2(\Omega_0),\rho_1 \in C^2(\Omega_1)$. 
    Let $T$ be the optimal Monge map with associated optimal plan 
    $\alpha\in\Gamma(\mu_0,\mu_1)$.
\end{itemize}
Let $\mu_t = (e_t)_\sharp \alpha$ and $v_t\in L^2_{\mu_t}(\R^m,\R^m)$ 
with $v_t\mu_t= (e_t)_\sharp(x_2-x_1)\alpha$ which then satisfy the continuity equation.
Then there is a Borel measurable representative $v_t$ such that there exists a solution of the flow equation
\begin{align}
    \frac{\d}{\d t}\phi_t&=v_t(\phi_t),\\
    \phi_0(x)&= x,
\end{align}
and $\mu_t=\phi_{t,\sharp}\mu_0$. 
Furthermore, we have
\[
v_t(\phi_t(x))=T(x)-x
\]
for $\mu_0$-a.e. $x \in \R^m$.
\end{proposition}

\begin{proof}
i): Let $\mu_0=\frac  1 n \sum_{i=1}^n\delta_{a_i}, \mu_1=\frac 1 n \sum_{i=1}^n\delta_{b_i}$ and let $T$ be a optimal map. The associated optimal plan is then $\alpha=\frac 1 n\sum_{i=1}^n\delta_{a_i,T(a_i)}$. Using $e_{t,\sharp}(x_2-x_1)\alpha=v_t\mu_t$ and $\mu_t=\frac 1 n \sum_{i=1}^n\delta_{T_t(a_i)}$ for $T_t(x)=(1-t)x+tT(x)$ we can conclude 
\[
v_t((1-t)a_i+tT(a_i))=T(a_i)-a_i.
\]
Furthermore, we have 
\begin{align}
\frac{\d}{\d t}T_t(a_i) = T(a_i)-a_i = v_t(T_t(a_i)),
\end{align}
and thus $\phi_t:=T_t$ fulfills the flow equation and $v_t(\phi_t(x))=T(x)-x$ for $\mu_0$-a.e. $x\in \R^m$.
\\[1ex]
ii):
First, note that by \citep[(16.12)]{ambrosio2021lectures} if there exists an invertible Monge map $T$ then the geodesic between $\mu_0,\mu_1$ fulfills the continuity equation with vector field 
\[
v_t= (T-\text{Id})\circ T_t^{-1}
\]
where $T_t=(1-t)\text{Id} + tT$.
By Caffarelli's regularity Theorem \citep[Theorem 12.50, ii)]{villani2009optimal}, we get the existence of a unique Monge map $T\in C^{1}(\Omega_0)$ mapping $\mu_0$ to $\mu_1$ and $U\in C^1(\Omega_1)$ mapping $\mu_1$ to $\mu_0$. By \citep[Theorem 5.2]{ambrosio2021lectures} we know that $T\circ U=\text{Id}$ on $\Omega_1$ and $U\circ T = \text{Id} $ on $\Omega_0$ and thus $T:\Omega_0\to\Omega_1$ is a $C^1$ diffeomorphism and in particular $\det(\nabla T)\neq 0$ on $\Omega_0$. Since we know by \citep[Proposition 6.2.12]{ambrosio2005gradient} that $\nabla T$ is positive definite $\mu_1$ a.e. on $\Omega_0$ we can deduce from $\det(\nabla T)\neq 0$ that $\nabla T$ is positive definite on $\Omega_0$. Consequently for $T_t=(1-t)\text{Id} + tT$ we have that $\nabla T_t=(1-t)\text{Id}+ t\nabla T$ is positive definite on 
$\Omega_0$ and thus the image of $\Omega_0$ under $T_t$ is open. Furthermore, we know by the proof of \citep[Proposition 6.2.12]{ambrosio2005gradient} that $T_t$ as a Monge map from $\mu_0$ to $\mu_t$ is injective on all points where $\nabla T_t$ is positive definite, which is on the whole $\Omega_0$, and thus $T_t $ is a diffeomorphism onto its image. Consequently, it possesses a $C^1$ inverse $T_t^{-1}:T_t(\Omega_0)\to \Omega_0$. Furthermore $(t,x)\mapsto (t,T_t(x))$ is an bijective Borel map from $[0,1]\times\Omega_0$ onto its Borel measurable image which we denote by $\Omega\subset [0,1]\times \R^m$. Thus $T_t^{-1}$ is a Borel measurable map from $\Omega\to \Omega_0$. Then for $v_t:= (T-\text{Id})\circ T_t^{-1}:T_t(\Omega_0)\to \R^m$ we have that $v_t$ is Borel measurable on $\Omega$ and also $\phi_t=T_t:\Omega_0\to \R^m$ is Borel measurable. 
Furthermore, we have
\begin{align}
\frac{\d}{\d t}\phi_t(x)=T(x)-x=(T-\text{Id})\circ T_t^{-1}(T_t(x))=v_t(\phi_t(x)).
\end{align}
Since we can set $\phi_t(x)=x$ on $\R^m\setminus \Omega_0$ and $v_t(x)=0$ for $x\in \R^m\setminus T_t(\Omega_0)$, we obtain the claim.
\end{proof}
\vspace{0.2cm}

\noindent
\textit{Proof of Proposition \ref{prop:flow_exists}.}
We will use the results from Proposition \ref{prop:empi_smooth_flow} and stack them with respect to $y_i$. The main obstruction is the measurability of the resulting objects which we address in the following.\\
For $e_t((y_1,x_1),(y_2,x_2))=(1-t)(y_1,x_1)+ t(y_2,x_2)$ and $\tilde{e}_t(x_1,x_2)=(1-t)x_1+tx_2$, it holds 
\begin{align}
\int_{\R^d \times \R^m}f(y,x) \, \d(e_t)_\sharp((y_2,x_2)-(y_1,x_1)\alpha)&=\int_{(\R^d \times \R^m)^2}f\circ e_t \cdot((y_2,x_2)-(y_1,x_1)) \, \d\alpha\\
&=\frac{1}{n}\sum_{i=1}^n\int_{\R^{2m}}f\circ e_t\cdot((y_i,x_1),(y_i,x_2))\, \d\alpha_{y_i}\\
&=\frac{1}{n}\sum_{i=1}^n\int_{\R^{2m}}f((y_i,\tilde{e}_t(x_1,x_2)))\, (0,x_2-x_1)\d\alpha_{y_i}\\
&=\frac{1}{n}\sum_{i=1}^n\int_{\R^{2m}}f\d(\tilde{e}_t)_\sharp(0,x_2-x_1)\alpha_{y_i}
\end{align}
and thus $(e_t)_\sharp((y_2,x_2)-(y_1,x_1)\alpha)=\frac  1 n\sum_{i=1}^n\delta_{y_i}\otimes \left(0,(\tilde{e}_t)_\sharp((x_2-x_1)\alpha_{y_i})\right)$. Combining with Proposition \ref{prop:continuity}, we conclude
\begin{align}
        v_t\mu_t&= \frac  1 n\sum_{i=1}^n\delta_{y_i}\otimes \left(0,(\tilde{e}_t)_\sharp((x_2-x_1)\alpha_{y_i})\right).
\end{align}
Furthermore, we have
\begin{align}
v_t\mu_t&=\int_{\R^d}v_t\d\mu_{t,y}\d P_Y=
    \frac 1 n\sum_{i=1}^n\delta_{y_i}\otimes v_t(y_i,\cdot)\mu_{t,y_i},
\end{align}
which implies $(\tilde{e}_t)_\sharp\left((x_2-x_1)\alpha_{y_i}\right)= \pi^2\circ(v_t(y_i,\cdot))\mu_{t,y_i}$ for all $i\in\{1,\ldots,n\}$. By Proposition \ref{prop:empi_smooth_flow} we know that there exists $\tilde{v}_{t,y_i}\in L^2(\mu_{t,y_i})$ with $(\tilde{e}_t)_\sharp\left(x_2-x_1))\alpha_{y_i}\right)= \tilde{v}_{t,y_i}(\cdot)\mu_{t,y_i}$ such that 
there exists a $\mu_{0,y_i}$-measurable solution $\phi_{t,y_i}$ of 
\begin{align}
\frac{\d}{\d t} \phi_{t,y_i} &= \tilde{v}_{t,y_i}\left(\phi_{t,y_i}\right)\\
\phi_{0,y_i}(x)&=x
\end{align}
for $\mu_{0,y_i}$ a.e. $x\in\R^m$ and $\mu_{t,y_i}=(\phi_{t,y_i})_\sharp\mu_{0,y_i}$.
Since $P_Y$ is a finite empirical measure also $\phi_t:\R^d \times \R^m\to \R^d \times \R^m$ 
defined on $(y_i,x) $ as $(y_i,\phi_{t,y_i}(x))$ is $\mu_t$ measurable and $\tilde{v}_t:(y_i,x)\mapsto (0,\tilde{v}_{t,y_i}(x))$ is in $L^2_{\mu_t}$ and coincides with $v_t$ as element of $L^2_{\mu_t}$. 
The latter is true since they coincide on $\{y_i\}\times \R^m$ up to a $\mu_{t,y_i}$ null set $\mathcal{N}_i$ because of
\[
\pi^2\circ(v_t(y_i,\cdot))\mu_{t,y_i}=(\tilde{e}_t)_\sharp\left((x_2-x_1)\alpha_{y_i}\right)=  \tilde{v}_{t,y_i}\mu_{t,y_i}.
\]
 Thus they coincide up to the set 
\[
\cup_{i=1}^n\{y_i\}\times \mathcal{N}_i\cup \{(y,x)\in \R^{d+m}:y\notin\{y_1,\ldots,y_n\}\}
\]
which is a $\mu_{t}$ null set.
Hence 
\begin{align}
\frac{\d}{\d t}\phi_t=\tilde{v}_t(\phi_t)
\end{align}
for $\mu_0$-a.e. $(y,x)\in\R^d \times \R^m$. Note that since $\tilde{v}_{t,y_i}$ is Borel measurable on $[0,1]\times \R^m$ we can assume that $\tilde{v}_t$ is Borel measurable on $[0,1]\times \R^d\times\R^m$ and similarly for $\phi_t$. Furthermore
\begin{align}
(\phi_t)_\sharp\mu_0(a\times b) &= \int_{(y,\phi_{t,y}(x))\in a\times b}\d \mu_0=\int_{y\in a}\int_{\phi_{t,y}(x)\in b}\d \mu_{0,y}(x)\d P_Y(y)\\
&= \int_{a}\int_{b}\d \phi_{t,y,\sharp}\mu_{0,y}\d P_Y(y) = \int_a\int_b\d \mu_{t,y}\d P_Y(y)\\
&= \mu_t(a\times b)
\end{align}
shows $\mu_t = (\phi_{t})_\sharp \mu_0$. The last claim follows from 
\begin{align}
    \tilde{v}_t((y_i,\phi_{t,y_i}(x))= (0,T_{y_i}(x)-x)
\end{align}
for $\mu_{0,y_i}$-a.e. $x\in\R^d$.
\hfill $\blacksquare$\\
\\[2ex]
\textit{Proof of Proposition \ref{prop:cont_Monge}}\\
In this paragraph we give a precise statement of Proposition \ref{prop:cont_Monge} as well as its proof. Recall that convex domain $\Omega \subset\R^n$ is called uniformely convex, if for every $\varepsilon > 0$ there exists a $\delta = \delta(\varepsilon) > 0$ such that for any two points $x,y \in \Omega$ satisfying $\| x - y \| \geq \varepsilon$, the midpoint $m = \frac{x + y}{2}$ 
fulfills $\mathrm{dist}\big(m, \partial \Omega\big) \geq \delta$. Here $\partial \Omega$ denotes the boundary of $\Omega$. Furthermore we say that a function $f:\Omega\to B$ for $\Omega\subset \R^n$ open and $B$ a Banach space, is a $C^1$ map if it is continously Frechet differentiable.\\
\textbf{Assumption 1.} We say that a measure $P_Y\in \P_2(\R^d)$ fulfills \emph{Assumption 1}, if there exists a uniformly convex bounded open $C^2$ subdomain $\Omega_Y\subseteq \R^d$ such that $P_Y(\Omega_Y)=1$ and it  admits a density $p_Y\in C^2(\Omega_Y)$ such that there exists $0< \delta< \epsilon$ such that $\delta\leq p_Y(y)\leq \epsilon$ for all $y\in\Omega_Y$.\\
\textbf{Assumption 2.} A measure $\mu\in \P_2(\R^d\times\R^m)$ is said to fulfill \emph{Assumption 2} if $\mu=P_Y\times \mu^Z$ for $\mu^Z\in\P_2(\R^m)$ a measure such that there exists a uniformly convex bounded open $C^2$ subdomain $\Omega_Z\subseteq \R^m$ such that $\mu_Z(\Omega_Z)=1$ and it  admits a density $p_Z\in C^2(\Omega_Z)$ such that there exists $0< \delta< \epsilon$ such that $\delta\leq p_Z(z)\leq \epsilon$ for all $z\in\Omega_Z$.\\
\textbf{Assumption 3} A measure $\mu\in P_{2,Y}(\R^d\times\R^m)$ is said to fulfill \emph{Assumption 3} if there exists a disintegration $\mu=\mu^y\times_y P_Y(y)$ such that there exists a uniformly convex bounded open $C^2$ subdomain $\Omega\subseteq \R^m$ such that $\mu^y(\Omega)=1$ and it  admits a density $p^y\in C^2(\Omega)$ such that there exists $0< \delta< \epsilon$ such that $\delta\leq p^y(x)\leq \epsilon$ for all $x\in\Omega$.\\
\begin{proposition}\label{prop:app_existence_ode}
Let $P_Y\in \P_2(\R^d)$ satisfy Assumption 1, $\mu_0=P_Y\times \mu^Z_0$ satisfy Assumption 2 and let $\mu_1=\mu_1^y\times_yP_Y$ satisfy Assumption 3 with density $p_1^y$ of $\mu_1^y$ on $\Omega$. Assume further that the map $y\mapsto p_1^y:\Omega_Y\to C^2(\Omega)$ be a $C^1$ map. Then there exists a $W_{2,Y}$-optimal transport map $T:(y,x)\mapsto (y,T_y(x))$ i.e. $\alpha=(\Id, T)_\sharp \mu_0\in \Gamma_{o,Y}(\mu_0,\mu_1)$ where $T_y$ is the optimal transport map for $\mu_0^Z$ and $\mu_1^y$.
Let $\mu_t = (e_t)_\sharp \alpha$ with associated vector field $v_t\in L^2_{\mu_t}$, where $(v_t)_j=0$ for all $j\leq d$. Then there is a representative of $v_t$ such that the flow equation
\begin{align}
\frac{\d}{\d t}\phi_t=v_t(\phi_t); \quad\phi_0(y,x)=(y,x)
\end{align}
admits a global solution and $\mu_t= (\phi_t)_\sharp \mu_0$. 
Furthermore, we have 
\begin{align}
v_t(\phi_t(y,x))=T(y,x)-(y,x)=\left(0,T_y(x)-x\right)
\end{align}
for $\mu_0$-a.e. $(y,x)\in \R^{d} \times \R^m$.
\end{proposition}
\begin{proof}
We will first construct a vector field which describes the inverse curve starting in $\mu_1$ and ending in $\mu_0$ first. Let $T^y$ be the $C^1$ Monge map between $\mu_1^y$ and $\mu_0^Z$, which exists and is unique by the Caffarelli's regularity Theorem, see \cite[Theorem 12.50, ii)]{villani2009optimal}. In order to use the latter theorem we need assumptions 2 and 3. Note that we have that $T_t^y(x)\coloneqq (1-t)x+tT^y(x)$ is a invertible $C^1$ map from $\Omega_1$ onto its image by the proof of Proposition \ref{prop:empi_smooth_flow} ii). Using assumptions 1-3 the assumption that $y\mapsto p_1^y$ is $C^1$, by \cite[Corollary 1.2]{gonzalez2024linearization} we have that $T(y,x)\coloneqq (y, T^y(x))$ is continuous. Thus we can conclude that $T$ is measurable and hence it is a Monge map with respect to $W_{2,Y}$ for $\mu_1$ and $\mu_0$. More precisely we have that $(\Id,T)_\sharp \mu_1\in \Gamma_{o,Y}(\mu_1,\mu_0)$ which follows e.g. from \eqref{relation}. \\
Claim: \emph{For $t\in[0,1]$ the map $T_t(y,x)=(y,T_t^y(x))$ as map $T_t:\Omega_1\times \Omega_Y\to \R^{d+m}$ is continuous and injective and its image, denoted by $O_t\subset \R^{d+m}$ is a Borel set.} The continuity follows from the continuity of $T$ and the injectivity from the injectivity of the individual $T_t^y$. The image of $\Omega_1\times \Omega_Y$ is Borel measurable as image of an open set under a injective continuous map.\\
Claim: \emph{Let $\nu_t=T_{t,\sharp}\mu_1$. Then $\nu_t=\mu_{1-t}$ and $O_t\subseteq \supp(\nu_t)$ as well as $\nu_t(O_t)=1$.} These claims follow from straigthforward computations\\
Claim: \emph{Let $O$ be the image of $(\Id,T_t):[0,1]\times\Omega_1\times \Omega_Y\to [0,1]\times\R^{d+m}$. Then $O$ is Borel measurable and we can define a Borel measurable map $T_t^{-1}:O\to \R^{d+m}$ which we can view as element in $L^2(\int \nu_t\d t,\R^{d+m})$}. This follows since $(\Id,T_t)$ is continuous and injective and thus has Borel measurable image. Furthermore it is injective which is why we can invert it on its image.\\
Claim: \emph{The map $u_t(y,x)\coloneqq (T-\Id)\circ T_t^{-1}$ is well defined as function in $L^2(\int \nu_t\d t)$.} This follows from above.\\
Claim: \emph{ For $\phi_t=T_t$ we have that $\frac{\d}{\d t}\phi_t(y,x)=u_t(\phi_t(y,x))$ and $$u_t\nu_t=e_{t,\sharp}\left(((y_2,x_2)-(y_1,x_1))(\Id,T)_\sharp\mu_1)\right).$$} Both claims can be verified from straightforward computations.\\
Claim: \emph{Then $v_t\coloneqq -u_{1-t}$ and $T^{-1}$ fulfill the claim.} Note that one can easily see that $T^{-1}$ is a Monge map for $\mu_0$ and $\mu_1$. Furthermore $T^{-1}(y,x)=(y,(T^y)^{-1})$ and $(T^y)^{-1}$ is a Monge map between $\mu_0^Z$ and $\mu_1^y$. By definition we have $\mu_t=(T^{-1})_{t,\sharp}\mu_0$ for $(T^{-1})_{t}(y,x)=(1-t)(y,x)+tT^{-1}(y,x)=(y,(1-t)x+t(T^y)^{-1}(x))$. It is easy to see that $(T^{-1})_t=T_{1-t}\circ T^{-1}$ and in turn we know that its image is a Borel set with unit mass under $\mu_t=\nu_{1-t}$ making $v_t\in L^2(\mu_t)$ well defined. Furthermore $\frac{\d}{\d t}(T^{-1})_t=T^{-1}-\Id$. Computing
\begin{align}
    v_t\circ(T^{-1})_t=-u_{1-t}\circ (T^{-1})_t=-(T-\Id)\circ T_{1-t}^{-1}\left(\circ T_{1-t}\circ T^{-1}\right)=T^{-1}-\Id
\end{align}
we can conclude the claim.
\end{proof}

\section{Proofs of Section \ref{sec:relax}}

\textit{Proof of Proposition \ref{prop_beta}.}
    Denote by $\alpha_{opt}$ the optimal transport plan associated to the conditional Wasserstein metric $W_{p,Y}$. Since it is only $Y$ diagonally supported, we have that
    \[
    \|(y_1,x_1)-(y_2,x_2)\|^p=d_\beta^p((y_1,x_1),(y_2,x_2))
    \]
    for $\alpha_{opt}$ a.e. $(y_1,x_1,y_2,x_2)\in (A\times B)^2$. Thus,
 for an optimal plan $\alpha$ for $W_{p,\beta}$, we conclude
\begin{align*}
    W_{p,Y}(\mu_0,\mu_1)^p
    &=\int_{(A\times B)^2}\|(y_1,x_1)-(y_2,x_2)\|^p \, \d\alpha_{opt}
    =\int_{(A\times B)^2}d_{\beta}^p ((y_1,x_1),(y_2,x_2)) \, \d\alpha_{opt} \\
                &\geq \int_{(A\times B)^2}d_{\beta}^p((y_1,x_1),(y_2,x_2)) \, \d\alpha\\
                &=\geq\int_{B^2}\|x_1-x_2\|^p \, \d\pi^{2,4}_{\sharp}\alpha
                + \beta\int_{A^2}\|y_1-y_2\|^p \, \d\pi^{1,3}_\sharp\alpha\\
                &\geq \beta\int_{A^2}\|y_1-y_2\|^p \, d\pi^{1,3}_\sharp\alpha
\end{align*}
and thus the claim. \hfill $\blacksquare$\\[1ex]
In order to proof Proposition \ref{prop:ex_fix} we need some auxiliary results, in particular the following lemma which is a variant of \citep[Proposition 7.1.3]{ambrosio2005gradient}.

\begin{lemma}\label{lemma:conv_sub_ex}
Let $\beta > 0$ and let $\mu_n\rightharpoonup\mu$, $\nu_n\rightharpoonup\nu$ with respect to weak convergence for $\mu_n,\nu_n,\mu,\nu\in \P_2(\R^d\times\R^m)$. Then there exists a subsequence of optimal plans $\alpha_{n_k}$ for $W_{2,d_\beta}(\mu_{n_k},\nu_{n_k})$ and an optimal plan $\alpha\in P_2\left((\R^d\times \R^m)^2\right)$ for $W_{2,d_\beta}(\mu,\nu)$ such that $\alpha_{n_k}\rightharpoonup \alpha$ weakly.
\end{lemma}
\begin{proof}
Let $f:\R^d\times \R^m\to \R^d\times \R^m$ be defined by $(y,x)\mapsto (\sqrt{\beta}y,x)$.  Then for $\mu_1,\mu_2\in P_{2}(\R^d\times \R^m)$ we have that $W_{2,d_\beta}(\mu_1,\mu_2)=W_2(f_\sharp\mu_1,f_\sharp\mu_2)$ since there is a bijection of couplings 
\begin{align}\label{eq:bi_couplings}
\alpha \mapsto (f,f)_\sharp \alpha
\end{align}
and we can compute
\begin{align}
\int_{(\R^d\times\R^m)^2}\|(y_2,x_2)-(y_1,x_2)\|^2\d (f,f)_\sharp \alpha=\int_{(\R^d\times \R^m)^2}\beta\|y_2-y_1\|^2+\|x_2-x_1\|\d\alpha
\end{align}
which implies that optimal couplings are mapped to optimal couplings. Since also $f_\sharp\mu_n\rightharpoonup f_\sharp\mu, f_\sharp\nu_n\rightharpoonup f_\sharp\nu$
we can use \citep[Proposition 7.1.3]{ambrosio2005gradient} to guarantee the existence of a subsequence of optimal plans $\tilde{\alpha}_{n_k}$ for $W_{2}(f_\sharp\mu_{n_k},f_\sharp\nu_{n_k})$ such that $\tilde{\alpha}_{n_k}\rightharpoonup \tilde{\alpha}$ for an optimal plan $\tilde{\alpha}$ for $W_{2}(f_\sharp\mu,f_\sharp\nu)$. Thus by \eqref{eq:bi_couplings} there exists a subsequence of optimal plans $\alpha_{n_k}$ for $W_{2,d_\beta}(\mu_n,\nu_n)$ such that $\alpha_{n_k}\rightharpoonup \alpha$ for an optimal plan $\alpha$ for $W_{2,d_\beta}(\mu,\nu)$. 
\end{proof}

\begin{remark}\label{rem:uni_weak}
A useful well known observation is the following. Let $\alpha_n\in \Gamma(\mu,\nu), n\in \N$ where $\mu,\nu\in\P_2(\R^l)$. Assume that we have weak convergence $\alpha_n \rightharpoonup \alpha$ for some $\alpha\in \P(\R^{2l})$. Then $\alpha\in\Gamma(\mu,\nu)$ and for any measurable function $f:\R^{2l}\to \R$ such that $|f(x_1,x_2)|\leq \|x_1\|^2+\|x_2\|^2$ we have that $\int_{\R^{2l}} f\d\alpha_n\to \int_{\R^{2l}} f\d\alpha$. The latter claim follows from \citep[Remark 5.2.3]{ambrosio2005gradient} which implies that $\|x_1\|^2+\|x_2\|^2$ is uniformly integrable w.r. to $\{\alpha_n,n\in\N\}$ since it has fixed marginals $\mu,\nu$ with finite second moments. Then \citep[Lemma 5.1.7]{ambrosio2005gradient}, which states that $\lim_{n\to\infty}\int f\d\alpha_n\to\int f\d\alpha$ for uniformly integrable $f$ and weak converging $\alpha_n\rightharpoonup \alpha$, implies the latter claim. It is immediate that $\alpha\in\Gamma(\mu,\nu)$ since $\Gamma(\mu,\nu)$ is compact and thus closed in the weak topology by \citep[Remark 5.2.3]{ambrosio2005gradient}.
\end{remark}
Note that \citep[Proposition 3.11]{hosseini2024conditional} states the following proposition only under some regularity conditions on $\mu,\nu$ which ensures uniqueness of optimal plans. But this is not needed in their proof if one is only interested in the existence of a optimal limit point. For the convenience of the reader we include a proof adapted to our situation but claim no originality.
\begin{proposition} \label{prop:huss}Let $\mu,\nu\in \P_{2,Y}(\R^d\times \R^m)$ and let $\beta_k\in\R^{\N}$ be a monotonly increasing series with $\beta_k\to \infty$. For a choice of optimal plans $\alpha^{\beta_k}$ for $W_{2,d_{\beta_k}}(\mu,\nu)$ the the closure of the set $\{\alpha^{\beta_k}\}_{k\in\N}$ in $\Gamma(\mu,\nu)$  is compact with respect to the weak convergence topology and every accumulation point $\alpha$ is a optimal plan for $W_{2,Y}(\mu,\nu)$.
\end{proposition}
\begin{proof}
By \citep[Remark 5.2.3]{ambrosio2005gradient} the set $\Gamma(\mu,\nu)$ is compact with respect to the weak convergence topology and thus also the closure of $\{\alpha^{\beta_k}:k\in\N\}$ is compact. Consequently accumulation points exist. Let $\alpha\in\Gamma(\mu,\nu)$ be an accumulation point and by abuse of notation let $\alpha^{\beta_k}\rightharpoonup \alpha\in \Gamma(\mu,\nu)$. Then by Remark \ref{rem:uni_weak} also 
\begin{align}
    \int_{\R^{2d+2m}}\|y_1-y_2\|^2\d\alpha^{\beta_k}\to \int_{\R^{2d+2m}}\|y_1-y_2\|^2\d\alpha=\int_{\R^{2d}}\|y_1-y_2\|^2\d\pi^{1,3}_\sharp\alpha.
\end{align}
Since the limes on the left is $0$ by Proposition \ref{prop_beta} we know that $\pi^{1,3}_\sharp\alpha$ is only supported on the diagonal. Thus for any $\pi^{1,3}_\sharp\alpha$ measurable bounded function $f:\R^{2d}\to \R$ we have that $f=f\circ \Delta\circ\pi^1$ for $\pi^{1,3}_\sharp\alpha$ a.e. $(y_1,y_2)\in\R^{2d}$ and hence
\begin{align}
\int_{\R^{2d}}f\d\pi^{1,3}_\sharp\alpha = \int_{\R^{2d}}f\circ \Delta\circ\pi^1\d\pi^{1,3}_\sharp\alpha = \int_{\R^{2d}}f\d\Delta_\sharp P_Y
\end{align} which implies $\pi^{1,3}_\sharp\alpha=\Delta_\sharp P_Y$ i.e. $\alpha\in\Gamma^4_{Y}(\mu,\nu)$. Furthermore note that $W_{2,\beta}\leq W_{2,Y}$ since every admissible plan for $W_{2,Y}$ is also admissible for $W_{2,d_\beta}$ with equal costs and thus 
\[
\int_{\R^{2d+2m}}\|(y_1,x_1)-(y_2,x_2)\|^2\d \alpha^{\beta_k} \leq W_{2,Y}(\mu,\nu)^2
\]
for all $k\in \N$. Using Remark \ref{rem:uni_weak} we can conclude that 
\[
\int_{\R^{2d+2m}}\|(y_1,x_1)-(y_2,x_2)\|^2\d \alpha=\lim_{k\to\infty}\int_{\R^{2d+2m}}\|(y_1,x_1)-(y_2,x_2)\|^2\d \alpha^{\beta_k} \leq W_{2,Y}(\mu,\nu)^2.
\]
Hence $\alpha$ is an optimal plan for $W_{2,Y}(\mu,\nu)$ and thus the claim. 
\end{proof}
    
\noindent
\textit{Proof of Proposition \ref{prop:ex_fix}.}
 By \citep[Theorem 8.3.2]{bogachev2007measure} we know that weak convergence of probability measures is metrizable and we denote by $d_{weak}$ a metric on $\P\left(\R^{2d}\times\R^{2m}\right)$ that metrizes weak convergence. Let $\alpha_n$ be a an optimal plan for $\mu_n,\nu_n$ and $W_{2,d_{\beta_k}}(\mu_n,\nu_n)$. Then by \citep[Proposition 7.1.3]{ambrosio2005gradient} and Lemma \ref{lemma:conv_sub_ex}, there exists a subsequence of $\alpha_{n}$ optimal for $W_{2,d_{\beta_k}}(\mu_n,\nu_n)$ converging weakly to an optimal plan $\alpha^{\beta_k}$ for $W_{2,d_{\beta_k}}(\mu,\nu)$. Thus we can find a sequence of optimal plans $\alpha_{n_k}$ for $W_{2,d_{\beta_k}}(\mu_{n_k},\nu_{n_k})$ such that $d_{weak}(\alpha^{\beta_k},\alpha_{n_k})<\frac 1 k$. We know by \citep[Proposition 3.11]{hosseini2024conditional} resp. Proposition \ref{prop:huss} that $\alpha_{\beta_k}\to \alpha$  with respect to $d_{weak}$ for an optimal plan $\alpha\in \Gamma^4_Y(\mu,\nu)$ for $W_{2,Y}(\mu,\nu)$. Thus, for $\epsilon >0$, there exists a $k$ such that $\frac{1}{k} + d_{weak}(\alpha^{\beta_k},\alpha)< \epsilon$ and we obtain 
\[
d_{weak}(\alpha_{n_k},\alpha)\leq d_{weak}(\alpha_{n_k},\alpha^{\beta_k})+d_{weak}(\alpha^{\beta_k},\alpha)\leq \frac{1}{k} + d_{weak}(\alpha^{\beta_k},\alpha)< \epsilon
\]
which proves the claim.
\hfill $\blacksquare$

\section{Implementation Details}
We use a setup similar to \citep{tong2023improving}, using the time dependent U-Net architecture from \citep{nichol2021improved, dhariwal2021diffusion} which are trained using Adam \citep{KB2015}. As in \citep{tong2023improving} we clip the gradient norm to 1 and use exponential moving averaging with a decay of $0.9999$. The differences are we use a constant learning rate of 2e-4, 256 model channels and no dropout. We train using 50k target samples for 500 epochs using a batch size of 500 for the minibatch OT couplings and a batch size of 100 for training the networks. We set the same random seed during training to be able to compare runs for different sources of couplings. 
The conditional coupling plans are calculated using the Python Optimal Transport package \citep{flamary2021pot}.  For inference simulate the corresponding ODEs using the torchiffeq \citep{torchdiffeq} package. To evaluate our results, we use the Fr\'echet inception distance (FID) \citep{HRUNH2017}\footnote{We use the implementation from \url{https://github.com/mseitzer/pytorch-fid}.}. We compute the distance on 50k training samples, for which we generate 50k samples given the same labels as the training samples. 

Further generated samples for the best performing method i.e $\beta = 100$:

\begin{figure}
    \centering
     \includegraphics[width=\textwidth]{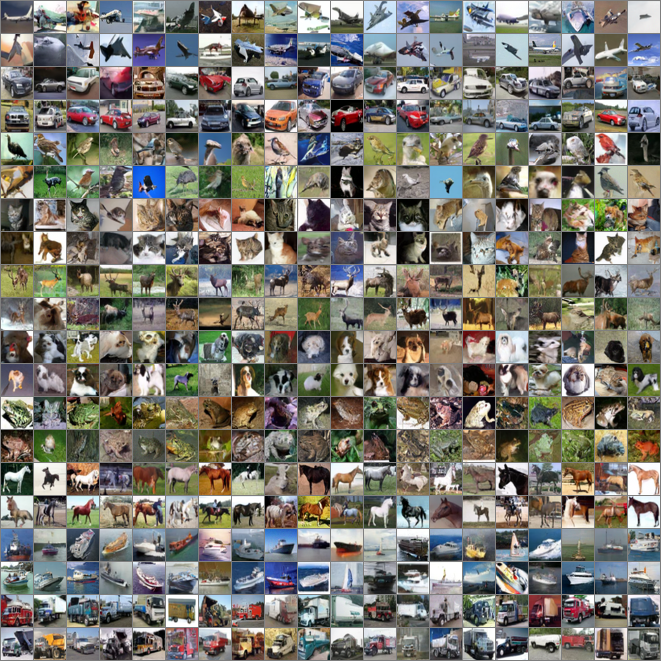}
    \caption{Uncurated samples sorted by class labels of the OT Bayesian Flow matching method with $\beta = 100$.}
\end{figure}

\newpage
\bibliography{references}

\end{document}